\documentclass[11pt]{article} 
\usepackage{geometry}
\usepackage{graphicx}
\graphicspath{{Figures/}}

\usepackage{amsthm,amsfonts,amssymb}
\usepackage{amssymb,mathrsfs}
\usepackage[cmex10]{amsmath}
\usepackage{bbm}

\usepackage{multirow}
\usepackage{tikz}
\usepackage{bm}
\usepackage{soul}

\usepackage{algorithm}
\usepackage{algpseudocode}
\algnewcommand\algorithmicinput{\textbf{Input:}}
\algnewcommand\INPUT{\item[\algorithmicinput]}
\algnewcommand\algorithmicoutput{\textbf{Output:}}
\algnewcommand\OUTPUT{\item[\algorithmicoutput]}
\def\mindex#1{\index{#1}}

\def\sq{\hbox{\rlap{$\sqcap$}$\sqcup$}}
\def\qed{\ifmmode\sq\else{\unskip\nobreak\hfil
\penalty50\hskip1em\null\nobreak\hfil\sq
\parfillskip=0pt\finalhyphendemerits=0\endgraf}\fi\medskip}

\long\def\defbox#1{\framebox[.9\hsize][c]{\parbox{.85\hsize}{%
\parindent=0pt
\baselineskip=12pt plus .1pt      
\parskip=6pt plus 1.5pt minus 1pt 
 #1}}}

\long\def\beginbox#1\endbox{\subsection*{}%
\hbox{\hspace{.05\hsize}\defbox{\medskip#1\bigskip}}%
\subsection*{}}

\def\endbox{}

\def\transpose{{\intercal}}

\def\darrow{\buildrel{\rm d}\over\longrightarrow}

\newsavebox{\junk}
\savebox{\junk}[1.6mm]{\hbox{$|\!|\!|$}}

\def\limsup{\mathop{\rm lim\ sup}}

\def\argmin{\mathop{\rm arg\, min}}

\def\U{{\sf U}}

\def\state{{\sf X}}

\newcommand{\field}[1]{\mathbb{#1}}

\def\Re{\field{R}}

\def\ind{\field{I}}

\def\Co{\field{C}}

\def\bfmath#1{{\mathchoice{\mbox{\boldmath$#1$}}%
{\mbox{\boldmath$#1$}}%
{\mbox{\boldmath$\scriptstyle#1$}}%
{\mbox{\boldmath$\scriptscriptstyle#1$}}}}

\def\bfmG{\bfmath{G}}     
     
\def\bfmI{\bfmath{I}}

\def\bfmU{\bfmath{U}}

\def\bfmX{\bfmath{X}}

\def\bfmY{\bfmath{Y}}

\def\bfmhhaY{\bfmath{\hhaY}} 
\def\bfmhhaY{\hbox to 0pt{$\widehat{\bfmY}$\hss}\widehat{\phantom{\raise 1.25pt\hbox{$\bfmY$}}}}

\def\til={{\widetilde =}}

\def\tilQ{{\widetilde Q}}
\def\tilT{\widetilde T}

\def\tilH{{\widetilde H}}

\def\tiltheta{\widetilde \theta}

\def\tiltheta{{\tilde \theta}}

\def\clA{{\cal A}}
\def\clB{{\cal B}}

\def\clE{{\cal E}}
\def\clF{{\cal F}}

\def\clN{{\cal N}}

 \def\FRAC#1#2#3{\genfrac{}{}{}{#1}{#2}{#3}}

\def\ddt{{\mathchoice{\FRAC{1}{d}{dt}}%
{\FRAC{1}{d}{dt}}%
{\FRAC{3}{d}{dt}}%
{\FRAC{3}{d}{dt}}}}

\def\ddtp{{\mathchoice{\FRAC{1}{d^{\hbox to 2pt{\rm\tiny +\hss}}}{dt}}%
{\FRAC{1}{d^{\hbox to 2pt{\rm\tiny +\hss}}}{dt}}%
{\FRAC{3}{d^{\hbox to 2pt{\rm\tiny +\hss}}}{dt}}%
{\FRAC{3}{d^{\hbox to 2pt{\rm\tiny +\hss}}}{dt}}}}

\def\half{{\mathchoice{\FRAC{1}{1}{2}}%
{\FRAC{1}{1}{2}}%
{\FRAC{3}{1}{2}}%
{\FRAC{3}{1}{2}}}}

\def\eqdef{\mathbin{:=}}

\def\Prob{{\sf P}}

\def\Expect{{\sf E}}

\def\average#1,#2,{{1\over #2} \sum_{#1}^{#2}}

\def\eye(#1){{\bf(#1)}\quad}
\def\epsy{\varepsilon}

\def\varble{\,\cdot\,}

\newtheorem{theorem}{Theorem}[section]

\newtheorem{proposition}[theorem]{Proposition}
\newtheorem{lemma}[theorem]{Lemma}

\def\Lemma#1{Lemma~\ref{#1}}
\def\Proposition#1{Prop.~\ref{#1}}
\def\Theorem#1{Theorem~\ref{#1}}

\def\Section#1{Section~\ref{#1}}

\def\Figure#1{Figure~\ref{#1}}

\def\Appendix#1{Appendix~\ref{#1}}

\def\eq#1/{(\ref{e:#1})}

\newcommand{\beqn}[1]{\notes{#1}%
\begin{eqnarray} \elabel{#1}}

\newcommand{\eeqn}{\end{eqnarray} }

\newcommand{\beq}[1]{\notes{#1}%
\begin{equation}\elabel{#1}}

\newcommand{\eeq}{\end{equation}}

\def\bdes{\begin{description}}
\def\edes{\end{description}}

\def\barf{{\overline {f}}}

\def\barn{{\overline {n}}}

\def\barF{\overline{\clF}}
\def\barsigma{\overline{\sigma}}

\newcounter{rmnum}
\newenvironment{romannum}{\begin{list}{{\upshape (\roman{rmnum})}}{\usecounter{rmnum}
			\setlength{\leftmargin}{12pt}
			\setlength{\rightmargin}{10pt}
			\setlength{\itemindent}{-1pt}
	}}{\end{list}}

\newcounter{anum}
\newenvironment{alphanum}{\begin{list}{{\upshape (\alph{anum})}}{\usecounter{anum}
\setlength{\leftmargin}{14pt}
\setlength{\rightmargin}{8pt}
\setlength{\itemsep}{2pt}
\setlength{\itemindent}{-1pt}
}}{\end{list}}

{\end{list}}

\def\ass(#1:#2){(#1\ref{#1:#2})}

\def\ritem#1{
\item[{\sf \ass(\current_model:#1)}]
}

\newenvironment{recall-ass}[1]{%
\begin{description}
\def\current_model{#1}}{
\end{description}
}

\newcommand{\bd}{\begin{description}}
\newcommand{\ed}{\end{description}}
\newcommand{\bt}{\begin{theorem}}
\newcommand{\et}{\end{theorem}}
\newcommand{\ba}{\begin{array}{rcl}}
\newcommand{\ea}{\end{array}}

\def\one{\mathbbm{1}}

\newcommand{\defn}[1]{{\protect\textit{#1}}\index{#1}}

\def\RateFn{\bar{I}}

\def\barb{\bar b}

\def\tilA{\tilde{A}}

\def\nd{\ell}  
\def\ninp{\ell_u}  
\def\nphi{\ell_\phi}  

\def\upmf{\mu}

\def\uq{\underline{q}}
\def\uQ{\underline{Q}}

\def\elig{\zeta}

\def\trace{\hbox{\rm trace\,}}  

\def\eqdef{\mathbin{:=}}

\def\tilh{\tilde h}
\def\Alg#1{Alg.~\ref{a:Zapalg}}

\def\assume#1{\smallbreak\noindent\textbf{#1}}

\def\nd{\ell}  
\def\ninp{\ell_u}  
\def\nphi{\ell_\phi}

\def\tilBE{{\widetilde \clB}}

\def\Real{\text{\rm Re}}

\def\upmf{\mu}

\def\bfmG{\bfmath{G}}

\def\bftheta{\bfmath{\theta}}

\def\state{{\sf X}}

\def\eqdef{\mathbin{:=}}

\def\elig{\zeta}
\def\bfelig{\bfmath{\zeta}}

\def\uq{\underline{q}}
\def\uQ{\underline{Q}}

\def\uh{\underline{h}}
\def\uH{\underline{H}}

\def\trace{\hbox{\rm trace\,}}

\def\Lemma#1{Lemma~\ref{#1}}
\def\Proposition#1{Prop.~\ref{#1}}
\def\Prop#1{Prop.~\ref{#1}}
\def\Theorem#1{Thm.~\ref{#1}}   

\def\RateFn{\bar{I}}

\def\barb{\bar b}

\def\tilA{\tilde{A}}

\def\clA{\mathcal{A}}

\def\haPi{{\widehat \Pi}}

\def\Fig#1{Fig.~\ref{#1}}

\def\ind{\field{I}}

\def\Re{\field{R}}

\def\diagpie{\Pi}

\def\diagpie{\Pi}
\def\pie{\varpi}

\def\barF{\bar{F}}

\usepackage[figbotcap]{subfigure}

\def\notes#1{}

\begin{document}
%
\title{Q-learning with Uniformly Bounded Variance}
\date{}
\maketitle
%
%
%

\vspace{-0.7in}

\author{
\begin{center}
Adithya~M.~Devraj\footnote{Department of EE, Stanford University, Stanford, CA-94305. Email: {\tt adevraj@stanford.edu}},~and~Sean~P.~Meyn\footnote{Department of ECE, University of Florida, Gainesville, FL-32611.
Email:~{\tt meyn@ece.ufl.edu}}%
\end{center} 
}

\markboth{Journal~XX ~Vol.~XX, No.~XX}%
{Devraj \& Meyn}
%




\begin{abstract}
	

Sample complexity bounds are a common performance metric in the  Reinforcement Learning literature. In the discounted cost, infinite horizon setting, all of the known bounds have a factor that is a polynomial in $1/(1-\gamma)$, where $\gamma < 1$ is the discount factor. For a large discount factor, these bounds seem to imply that a very large number of samples is required to achieve an $\epsy$-optimal policy. The objective of the present work is to introduce a new class of   algorithms that have   sample complexity \emph{uniformly bounded for all $\gamma < 1$}. One may argue that this is impossible, due to a recent min-max lower bound.  The explanation is that this previous lower bound is for a specific problem, which we modify, without compromising the ultimate objective of obtaining an $\epsy$-optimal policy. Specifically, we show that the asymptotic covariance of the Q-learning algorithm with an optimized step-size sequence is a quadratic function of $1/(1-\gamma)$;
an expected, and essentially known result. The new \emph{relative Q-learning} algorithm proposed here is shown to have asymptotic covariance that is a quadratic in $1/(1- \rho^* \gamma)$,   where  $1 - \rho^* > 0$ is an \emph{upper bound} on the spectral gap of an optimal transition matrix.

\end{abstract}


{
\textbf{Acknowledgements:}
Financial support from ARO award W911NF1810334
and National Science Foundation award  EPCN 1935389
is gratefully acknowledged.
}



%
%
%
%

\section{Introduction}
\label{sec:intro}

Many Reinforcement Learning (RL) algorithms can be cast as parameter estimation techniques, where the goal is to recursively estimate the parameter vector $\theta^* \in \Re^d$ that directly, or indirectly yields an optimal decision making rule within a parameterized family. In these algorithms, the update equation for the $d$-dimensional parameter estimates  $\{\theta_n: n \geq 0\}$ can be expressed in the general form
\begin{equation}
\theta_{n+1} = \theta_n +\alpha_{n+1} [ \barf(\theta_n) + \Delta_{n+1} ]\,,\quad n\ge 0
\label{e:SAintro}
\end{equation}
in which $\theta_0 \in \Re^d$ is given, $\{  \alpha_n \}$ is a positive scalar \defn{gain sequence} (also known as \defn{learning rate}), $\barf: \Re^d \to \Re^d$ is a deterministic function, and $\{  \Delta_n \}$ is a ``noise'' sequence.   

The recursion \eqref{e:SAintro} is an example of stochastic approximation (SA), for which there is a vast research literature.
Under standard assumptions, it can be shown that $
\lim_{n \to \infty} \theta_n = \theta^*
$, where the limit satisfies $\barf(\theta^*) = 0$.
Moreover, it can be shown that the  best algorithms  achieve the optimal mean-square error (MSE) convergence rate:  
\begin{equation}
\Expect \big[ \| \theta_n - \theta^* \|^2 \big] =  O(1/n)  
\label{e:mse_roc}
\end{equation}

It is known that TD- and Q-learning can be written in the form \eqref{e:SAintro} \cite{tsiroy97a,bormey00a,devmey17b}. In these algorithms, $\{\theta_n\}$ represents the sequence of parameter estimates that are used to approximate a \defn{value function} or \defn{Q-function}.


It was first established in our work \cite{devmey17a,devmey17b} that the convergence rate of the MSE of Watkins' Q-learning (i.e., Q-learning with a tabular basis) can be as slow as $O(1/n^{2(1-\gamma)})$, if the discount factor $\gamma\in (0, 1)$ satisfies $\gamma > \half$, and if the step-size $\alpha_n$ is either of two standard forms (see discussion in \Section{sec:q:watkins}). It was also shown that the optimal convergence rate \eqref{e:mse_roc} is obtained by using a step-size of the form $\alpha_n = g/n$, where $g$ is a scalar proportional to $1/(1-\gamma)$; this is consistent with conclusions in more recent research \cite{wai19a,quwie20}.
In the earlier work \cite{sze97}, a sample path \textit{upper bound} was obtained on the rate of convergence that is roughly consistent with the mean-square rate established for $\gamma >\half$ in \cite{devmey17a,devmey17b}.

 
Since the publication of  \cite{sze97},  many papers have appeared with proposed improvements to the algorithm;  often including (non-asymptotic) finite-$n$ bounds   on the MSE \eqref{e:mse_roc}. Ignoring higher order terms, 
these bounds can be expressed in the following general form \cite{azamunghakap11,wai19a,wai19v,quwie20,strehl2006pac}:
\begin{equation}
\Expect \big[ \| \theta_n - \theta^* \|^2 \big] \leq \frac{1}{(1-\gamma)^p} \cdot \frac{B}{n}
\label{e:poly_discount}
\end{equation}
where $p \geq 2$ is a scalar.  The constant
$B$ is a function of the total number of state-action pairs, the discount factor $\gamma$, and 
the maximum per-time-step cost. Much of the literature has worked towards minimizing $p$ through a combination of hard analysis and algorithm design.  



It is widely observed that Q-leanring algorithms can be very slow to converge, especially when the discount factor is close to $1$; 
The bound in \eqref{e:poly_discount} offers an explanation for this phenomenon. 
Quoting \cite{azamunghakap11},  a primary reason for slow convergence is ``\textit{the fact that the Bellman operator propagates information throughout the whole space}'',   especially when the discount factor is close to $1$.     We do not dispute these explanations, but in this paper we show that the challenge presented by large discounting is relatively minor.  In order to make this point clear we must  take a step back and rethink fundamentals: 
\begin{center}
\emph{Why do we need to estimate the Q-function?}
\end{center}

Letting $Q^*(x,u)$ denote the optimal Q-function evaluated at the state-action pair $(x,u)$, the main reason for estimating the Q-function is to obtain from it the corresponding optimal policy:
\begin{equation*}
\phi^*(x) \eqdef \argmin_u Q^*(x,u)
\label{e:phi_star}
\end{equation*}
It is clear from the above definition that adding a constant to $Q^*$ \emph{will not} alter $\phi^*$.
This is a fortunate fact: 
it is well-known that $Q^*$ can be decomposed as (see for example \cite{put14,ber95,ber12a}:
\begin{equation}
Q^*(x,u) = \tilQ^*(x,u) + \frac{\eta^*}{1-\gamma}  
\label{e:tilQ}
\end{equation}
where the scalar $\eta^*$ denotes the optimal average cost (independent of $(x,u)$ under the assumptions imposed here), 
and $\tilQ^*(x,u)$ is uniformly bounded in $\gamma$, $x$, and $u$.



The reason for slow performance of Q-learning when $\gamma\approx 1$ is because of the high variance in the indirect estimate of the large constant $\eta^* / (1-\gamma)$.  
It is argued in \Section{sec:RelQ} that if the error in the constants is ignored, a far tamer bound is obtained:
\begin{equation}
\Expect \big[ \| \theta_n - \theta^* \|^2 \big] \leq \frac{1}{(1- \rho^* \gamma)^p} \cdot \frac{B}{n}
\label{e:poly_gamma_discount}
\end{equation}
where $\rho^* <1$, and $1-\rho^*$ is an \emph{upper bound} on the spectral gap of the transition matrix for  the pair process $(\bfmX,\bfmU)$ under the optimal policy (details are in \Section{sec:ConvRateRelQ}).


\vspace{-0.03in}

\begin{figure}[htbp]
	\centering
	\includegraphics[width=\hsize]{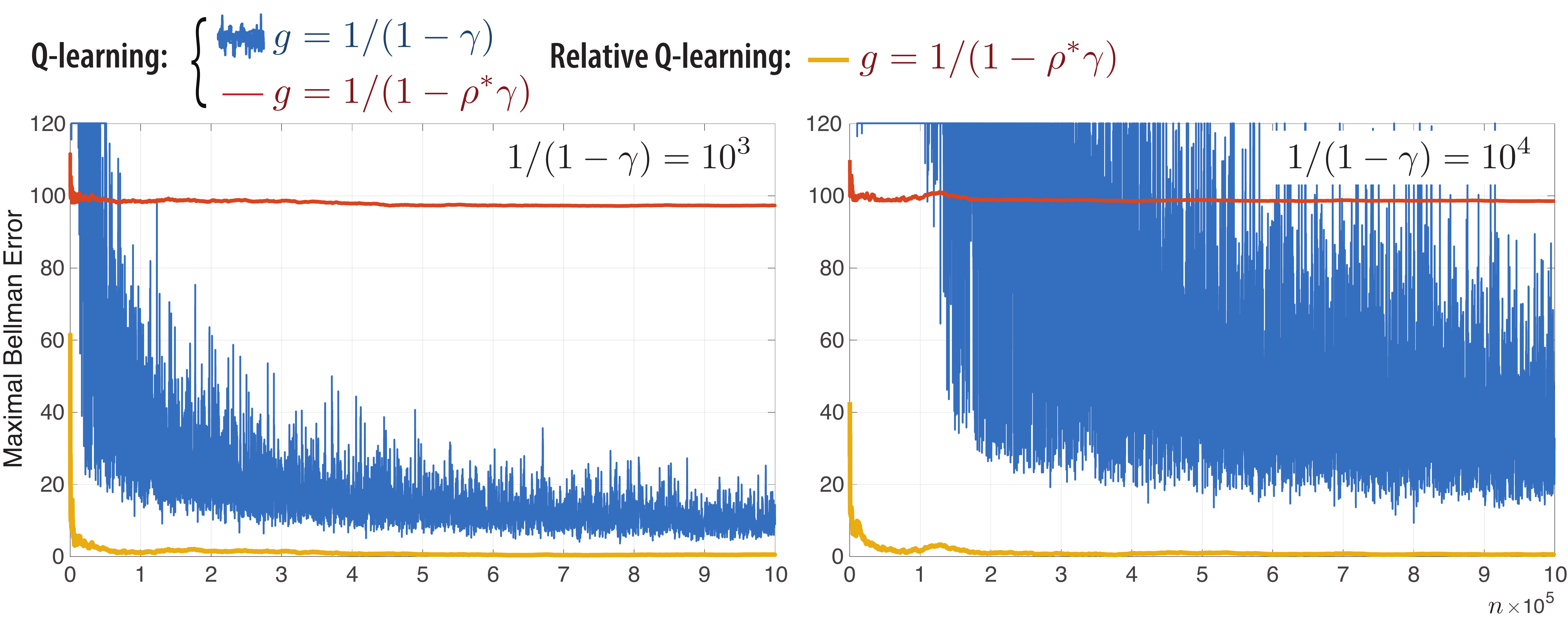}
	
	\caption{Comparison of Q-learning and Relative Q-learning algorithms for the stochastic shortest path problem of \cite{devmey17a}. The relative Q-learning algorithm is unaffected by large discounting.}
	\label{f:Hplot}
\end{figure}

The new \emph{relative Q-learning} algorithm proposed here is designed to achieve the upper bound   \eqref{e:poly_gamma_discount}.   Unfortunately, we have not yet obtained this explicit finite-$n$ bound.   We have instead obtained formulae for the \emph{asymptotic covariance} that corresponds to each of the algorithms considered in this paper (see \eqref{e:SAsigma}).

The close relationship between  the asymptotic covariance and sample complexity bounds is discussed in \Section{sec:sample_complexity}, based on the theoretical background in \Section{sec:sa_rl}.


\vspace{-0.03in}

\subsection{Stochastic Approximation \& Reinforcement Learning}
\label{sec:sa_rl}

Consider a parameterized family of  $\Re^d$-valued functions $\{\barf(\theta): \theta \in \Re^d \}$ that can be expressed as an expectation,
\begin{equation}
\barf(\theta) \eqdef \Expect \big[ f(\theta, \Phi) \big]  \,,\qquad \theta \in \Re^d\,,
\label{e:barf_theta}
\end{equation}
with $\Phi \in \Re^m$ a random vector, $f: \Re^d \times \Re^m \to \Re^d$, and the expectation is with respect to the distribution of the random vector $\Phi$.
It is assumed throughout that there exists a unique vector $\theta^*\in \Re^d$ satisfying $\barf(\theta^*)=0$. Under this assumption, the goal of SA is to estimate $\theta^*$.

The sequence of estimates obtained from the SA algorithm are defined as follows:
\begin{equation} 
\theta_{n+1}= \theta_n + \alpha_{n + 1} f(\theta_n \,, \Phi_{n+1} )
\label{e:SAa}
\end{equation}
where $\theta_0\,\in\Re^d $ is given, $\Phi_n$ has the same distribution as $\Phi$ for each $n \geq 0$  (or its distribution converges to that of $\Phi$ as $n\to\infty$), and $\{\alpha_n\}$ is a non-negative scalar step-size sequence. We assume $\alpha_n = g / n$ for some scalar $g > 0$, and special cases in applications to Q-learning are discussed separately in \Section{sec:q}.

Asymptotic statistical theory for SA  is extremely rich.  Large Deviations or Central Limit Theorem (CLT) limits hold under very general assumptions for both SA and related Monte-Carlo techniques \cite{benmetpri12,kusyin97,kon02,bor20a,MT}.



The CLT will guide  design and analysis of algorithms in this paper.  For a typical SA algorithm,  this takes the following form:  Denote the \defn{error sequence}  by
\begin{equation}
\tiltheta_n \eqdef \theta_n-\theta^*
\end{equation}
Under general conditions, the CLT states that the scaled sequence $\{\sqrt{n} \tiltheta_n : n\ge 0\}$	converges in distribution to a multivariate Gaussian $\clN (0,\Sigma_\theta)$.   Typically,  the covariance  matrix of this scaled sequence is also convergent:
\begin{equation}
\Sigma_\theta=\lim_{n\to\infty} n \Expect[\tiltheta_n\tiltheta_n^\transpose]
\label{e:SAsigma}
\end{equation}  
The limit $\Sigma_\theta$ is known as the \textit{asymptotic covariance}.  
Provided it is finite, \eqref{e:SAsigma} implies \eqref{e:mse_roc}, which is the fastest possible rate \cite{benmetpri12,kusyin97,bor20a,robmon51a,rup85}.  
For Q-learning, this also implies a bound of the form \eqref{e:poly_discount}, but for $n$ \emph{``large enough''}.

An asymptotic bound such as \eqref{e:SAsigma} may not be satisfying for RL practitioners, given the success of finite-time performance bounds in prior research.    There are however good reasons to apply this asymptotic theory in algorithm design:
\begin{romannum}
\item 
The asymptotic covariance $\Sigma_\theta$ has a simple representation as the solution to a Lyapunov equation.

\item 
The MSE convergence is refined in  \cite{chedevbusmey20} for \emph{linear} SA algorithms (see \Section{sec:mse_lin_SA}):  For some $\delta > 0$,
\begin{equation}
\Sigma_n = n^{-1} \Sigma_\theta + O(n^{-1- \delta})\,, \quad \text{where, }\,\Sigma_n \eqdef \Expect[\tiltheta_n \tiltheta_n^\transpose ]
\label{e:SigmaN}
\end{equation} 
{It is expected that these bounds can be extended to many nonlinear algorithms found in RL.}

\item  The asymptotic covariance lies beneath the surface of the theory of finite-time error bounds.  
Here is what can be expected from the theory of large deviations    \cite{demzei98a,konmey03a}, for which the \defn{rate function} is denoted
\begin{equation}
I_i(\epsy)\eqdef 
-
\lim_{n\to\infty} \frac{1}{n} \log \Prob\{  | \theta_n(i) - \theta^*(i) | > \epsy \} 
\label{e:PACasy}
\end{equation}
The second order Taylor series approximation  holds under general conditions:
\begin{equation}
I_i(\epsy) = \frac{1}{ 2 \sigma^2_\theta(i)  }   \epsy^2+O(\epsy^3)
\label{e:PACasyQuad}
\end{equation}    
where $\sigma^2_\theta(i) =  \Sigma_\theta(i,i) $, 
from which we obtain
\begin{equation}
\begin{aligned}
\Prob \{  | \theta_n(i) &  - \theta^*(i) |  > \epsy \}  
\\
& = \exp\Bigl\{  - \frac{\epsy^2 n}{ 2 \sigma^2_\theta(i) }    +O(n\epsy^3)  + o(n) \Bigr\}
\label{e:PACasyQuadCor}
\end{aligned}
\end{equation}
where $o(n)/n\to 0$ as $n\to\infty$,   and $ O(n\epsy^3) / n$  is bounded in $n\ge 1$,   and absolutely bounded by a constant times $\epsy^3$ for small $\epsy>0$.

\item 
The Central Limit Theorem (CLT) holds under general assumptions:
\begin{equation}
\sqrt{n} \tiltheta_n   \darrow  W
\label{e:SACLT}
\end{equation}
where the convergence is in distribution, and 
where $W$ is Gaussian $\clN(0,\Sigma_\theta)$ \cite{kusyin97,benmetpri12};    a version of the Law of the Iterated Logarithm (LIL) also holds \cite{mokpel05}:
\begin{equation*}
\lim\limits_{n \to \infty}\sqrt{\frac{n}{2 \log \log n}} \tiltheta_n  \in C
\end{equation*}
where $C = \{v \in \Re^d : v^\transpose \Sigma_{\theta}^{-1} v \leq 1 \}$.
An immediate corollary is \cite{kovsch03}:
\begin{equation}
\limsup\limits_{n \to \infty}\sqrt{\frac{n}{2 \log \log n}} \| \tiltheta_n \|  = \sqrt{\lambda_{\max} \big(\Sigma_\theta\big)}
\label{e:LIL-2}
\end{equation}
\end{romannum}
The asymptotic theory provides insight into the slow convergence of Watkins' Q-learning algorithm, and motivates better algorithms such as Zap~Q-learning \cite{devmey17b}, and the relative Q-learning algorithm introduced in \Section{sec:RelQ}.



\begin{figure*}
\centering
\includegraphics[width=1\textwidth]{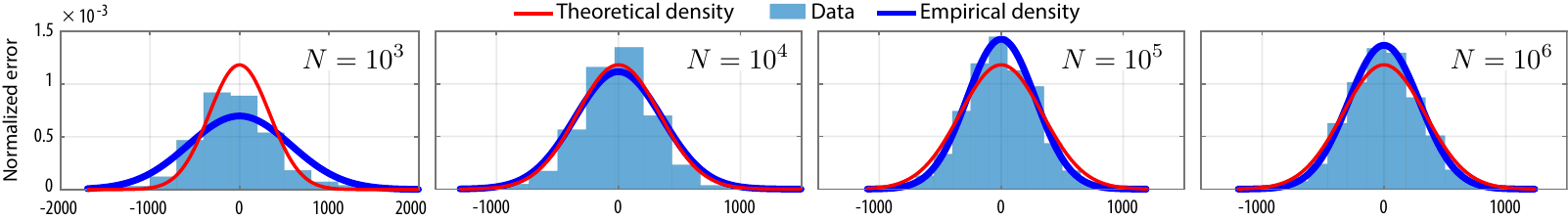}

\caption{Histogram of $\{ \sqrt{N} \tiltheta_N(i) \}$ for $10^3$ independent runs.   The CLT approximation is good even for the shortest run, and nearly perfect for $N\ge 10^4$.
}
\label{f:relW10CLT}
\vspace{-0.3cm}

\end{figure*}

\subsection{Sample complexity bounds}    
\label{sec:sample_complexity}

A sample complexity bound for an algorithm is defined based on the number of iterations required to obtain a desired probability of error.    Consider for concreteness a single entry $i$ of a parameter estimate in SA:  for given $\delta,\epsy >0$,   
we seek an integer $\barn_i(\epsy,\delta)$ such that
\begin{equation}
\Prob\{  | \theta_n(i) - \theta^* (i) | > \epsy \}   \le \delta\,, \quad \text{
for all $n\ge \barn_i(\epsy,\delta)$.}
\label{e:sc}
\end{equation}
Such bounds are a foundation of   statistical learning theory  \cite{devgyolaslug13}.
Below are three techniques to construct $\barn$, beginning with the most common approach:

\textbf{1. LDP theory}
The inequalities of Hoeffding and Bennett are finite-$n$ variants of \eqref{e:PACasy}:
\begin{equation}
\Prob\{  | \theta_n(i) - \theta^* (i) | > \epsy \}  \le \barb \exp(- n \RateFn_i(\epsy))  \,,  \ \ n\ge 1 
\label{e:PAC}
\end{equation}
where $\barb$ is a constant and $\RateFn_i (\epsy) >0$  for $\epsy>0$.   
A sample complexity bound then follows easily, with 
\begin{equation} 
\barn_i(\epsy,\delta) = \frac{1}{  \RateFn_i(\epsy) } \bigl[ \log(\barb) + \log(\delta^{-1}) ]
\label{e:PAC2}
\end{equation} 
See for example \cite{kak03,lathutsun13,azamunkap12,eveman03}, and \cite{glyorm02,konlasmey05a} for general theory in a Markov setting.

\textbf{2. MSE}    Given a true finite-$n$ version of \eqref{e:SigmaN}: 
\begin{equation}
\Expect[  ( \theta_n(i) - \theta^* (i) )^2 ]  \le  \barsigma^2(i)  n^{-1}
\label{e:finite-n-MSE}
\end{equation}
A sample complexity bound follows from
Chebyshev's inequality, using
\begin{equation} 
\barn_i(\epsy,\delta) = \frac{\barsigma^2(i)}{  \epsy^2}  \delta^{-1} 
\label{e:PAC2MSE}
\end{equation} 
Finite-$n$ bounds  on mean-square error are contained  in \cite{quwie20,sriyin19,chedevbusmey20,chemagshasha20},
and the mean $\ell_\infty $ bound in  \cite{wai19a} implies a similar sample complexity bound.



\textbf{3. CLT}   A finite-$n$ version of the CLT is the Berry-Esseen bound:  for all $z>0$,
\begin{equation}
\Big |  \varrho_i(z, n)
-   2 \bar{F}(z)  \Bigr| 
\le   \frac{K_i}{ \sqrt{n} }
\label{e:B-E}
\end{equation}
where
$ \varrho_i(z, n)$ is the error probability with CLT scaling:
\[
\varrho_i(z, n) =
\Prob\bigl\{ \sqrt{n}  | \theta_n(i) - \theta^* (i) | > z  \sigma_\theta(i)    \bigr \}  
\]
and $\bar{F}$ is the complementary CDF for a standard Normal r.v..  For i.i.d.\ sequences, a simple expression for   $K_i$ is available;   bounds for Markov sequences is less complete \cite{bol80,klo19}.





For any $z>0$ and $\delta >  2\bar F(z)$, denote
\begin{equation}
\barn_i(\epsy,\delta,z)    =   \max\Bigl\{   \frac{z^2}{\epsy^2} \sigma^2_\theta(i)  \, ,   \frac{1}{4}  \frac{ K_i^2}{ [\delta- 2\bar F(z) ]^2}  \Bigr\}
\label{e:sc_B-Ez}
\end{equation}
The bound \eqref{e:B-E} implies a family of sample complexity bounds that can be optimized over $z$: for $n\ge \barn_i(\epsy,\delta,z) $,
\begin{equation}
\begin{aligned}
\Prob\{  | \theta_n(i) - \theta^* (i) | > \epsy \} 
&  \le    2 \bar{F}(z) + 2  \frac{K_i}{ \sqrt{n} }   \le \delta
\end{aligned} 
\label{e:sc_B-E}
\end{equation}


The asymptotic covariance is central to each approach: 

\noindent\textbf{1.}  If the the limit \eqref{e:PACasy} and the bound  \eqref{e:PAC} each hold, then the rate function must dominate:  $ \RateFn_i (\epsy) \leq I_i(\epsy)$.   To maximize this upper bound we must  minimize $\sigma^2_\theta(i)$
(recall \eqref{e:PACasyQuad}, and remember we are typically interested in small $\epsy>0$).  

\noindent\textbf{2.}   Similarly, the mean-square error bound \eqref{e:finite-n-MSE} combined with the approximation \eqref{e:SigmaN} implies $\barsigma^2(i) \ge  \sigma^2_\theta(i)$.

\noindent\textbf{3.} The bound \eqref{e:sc_B-E} requires  $\sigma^2_\theta(i)$ through the definition \eqref{e:sc_B-Ez}.

This theory provides strong motivation for considering the asymptotic covariance $\Sigma_\theta$ in algorithm design.



 
Based on the above discussion, we conjecture that 
\begin{romannum}
\item The sample-path complexity bound \eqref{e:PAC2} with $\RateFn$ quadratic is possible for Watkins' algorithm, provided  we use $\alpha_n = [{1+(1-\gamma)n}]^{-1}$ in the right hand side of the update equation \eqref{e:SAintro}. This step-size was proposed independently in \cite{wai19a,devbusmey19}.
\item With relative Q-learning, we can obtain similar sample complexity result with $\alpha_n = [{1+(1-\rho^*)n}]^{-1}$, which is \emph{independent of $\gamma$}. 
\end{romannum} 
However, for more complex algorithms we do not expect to obtain tight bounds, with  $ \RateFn_i (\epsy) \approx I_i(\epsy)$.  For this reason we advocate the CLT for algorithm design and evaluation, even without a sharp Berry-Esseen bound.  We frequently find that the CLT is highly predictive of parameter error,  where the   covariance $\sigma^2_\theta(i)$  is estimated via independent runs.  	\Fig{f:relW10CLT}  shows results from one experiment using the relative Q-learning algorithm:  the histograms were obtained based on $10^3$ independent runs,  with time horizons ranging from $N=10^3$ to $10^6$.   The CLT approximation is good even for the shortest run, and nearly perfect for $N\ge 10^4$. 



\subsection{Explicit Mean Square Error bounds for SA}
\label{sec:mse_lin_SA}

We first present a special case of the main result of \cite{chedevbusmey20} for linear SA algorithms, and then an extension to nonlinear SA. These results are later recalled in applications to Q-learning. 

The analysis of the SA recursion \eqref{e:SAa} begins with the transformation to \eqref{e:SAintro},
with $ \Delta_{n+1} =  f(\theta_n \,, \Phi_{n+1} ) - \barf(\theta_n) $.
The difference $f(\theta \,, \Phi_{n+1} ) - \barf(\theta) $ has zero mean for any (deterministic) $\theta\in\Re^d$ when 
$\Phi_{n+1}$ has the same distribution as $\Phi$   (recall \eqref{e:barf_theta}). Though the results of \cite{chedevbusmey20} extend to Markovian noise, for the purposes of this paper, we assume that $\{\Delta_n\}$ is a martingale difference sequence:
\begin{romannum}
\item[\textbf{(A1)}]
The sequence $\{\Delta_n: n\ge 1\}$ is a martingale difference sequence. 
Moreover, for some $\bar\sigma^2_\Delta< \infty $ and any initial condition $\theta_0 \in \Re ^d$, 
\[
\Expect [\|\Delta_{n+1}\|^2\mid \Delta_1,\dots,\Delta_n] \leq \bar\sigma^2_\Delta (1 + \| \theta_n \|^2), 
\qquad n \geq 0
\]
\end{romannum}
We also assume a scalar, diminishing step-size sequence:
\begin{romannum}
\item[\textbf{(A2)}] $ \alpha_n  = g/n$, for some scalar $g >0$, and all $n \geq 1$
\end{romannum}

With $\Sigma_n$ defined in \eqref{e:SigmaN}, denote 
\[
\sigma_n^2 =\trace(\Sigma_n )= \Expect[ \| \tiltheta_n \|^2]
\]      
We say $\sigma_n^2 \to 0$ at rate $1/n^\mu$  (with $\mu>0$), if for each $\epsy>0$,
\begin{equation}
\lim_{n\to\infty}  n^{\mu-\epsy} \sigma_n^2 = 0  \qquad  \textit{and} 
\qquad
\lim_{n\to\infty}  n^{\mu+\epsy} \sigma_n^2 = \infty
\label{e:conv_rate}
\end{equation}
It is known that the maximal value is $\mu =1$.


The analysis in \cite{chedevbusmey20} is based on a ``linearized'' approximation of the SA recursion \eqref{e:SAa}:
\begin{equation}
\theta_{n+1} = \theta_n +  \alpha_{n+1} \bigl[A_{n+1}\theta_n - b_{n+1} \bigr]
\label{e:linearSA}
\end{equation}
where, $A_{n+1} = \clA(\Phi_{n+1})$ is a $d \times d$ matrix, and $b_{n+1} = b(\Phi_{n+1})$ is   $d \times 1$. 
Let $A$ and $b$ denote the respective means:  
\begin{equation}
A = \Expect[\clA(\Phi)]\,,\qquad 
b = \Expect[b(\Phi)]
\label{e:Ab}
\end{equation} 
where the expectations are in steady state.
We assume that the $d\times d$ matrix $A$ is Hurwitz,  a necessary condition for convergence of \eqref{e:linearSA}:
\begin{romannum}
\item[{\textbf{(A3)}}] The $d\times d$ matrix $A$ is Hurwitz. 
\qed
\end{romannum} 
(A3) implies that $A$ is invertible, and   $\theta^* = A^{-1}b$.

The recursion \eqref{e:linearSA} can be rewritten in the form \eqref{e:SAintro}:
\begin{equation}
\theta_{n+1} = \theta_n +  \alpha_{n+1} \bigl[A \theta_n - b + \Delta_{n+1} \bigr]
\label{e:linearSA_linearized}
\end{equation}
in which $\{\Delta_n\}$ is the noise sequence: 
\begin{equation}
\begin{aligned} 
\Delta_{n+1} = A_{n+1} \theta^* - b_{n+1} +  \tilA_{n+1} \tiltheta_n
\end{aligned} 
\label{e:Dn_one}
\end{equation}
with $\tilA_{n+1}=A_{n+1} - A$.
The parameter error sequence also evolves as a simple linear recursion:
\begin{equation}
\tiltheta_{n+1} = \tiltheta_n +  \alpha_{n+1} [A \tiltheta_n +  \Delta_{n+1}   ]
\label{e:linSA_MC}
\end{equation}

The asymptotic covariance \eqref{e:SAsigma} exists under special conditions (see \Theorem{t:SAlinearized}), and under these conditions it satisfies the Lyapunov equation
\begin{equation}
(g A+\half I)  \Sigma_\theta + \Sigma_\theta (g A+\half I) ^\transpose  + g^2\Sigma_\Delta  =0
\label{e:Lyapg}
\end{equation}
where the ``noise covariance matrix'' $\Sigma_\Delta$  is defined to be
\begin{equation}
\Sigma_\Delta 
= \Expect\big [ \big(A_{n+1} \theta^* - b_{n+1} \big) \big(A_{n+1} \theta^* - b_{n+1} \big)^\transpose ]
\label{e:SigmaDelta}
\end{equation}

\Theorem{t:SAlinearized} is a special case of the main result of  \cite{chedevbusmey20} (which does not impose the martingale  
assumption (A1)). 
\begin{theorem}
\label{t:SAlinearized}
Suppose {(A1)} -- {(A3)} hold. Then the following hold for the linear recursion \eqref{e:linSA_MC}, for each initial   $(\Phi_0, \tiltheta_{0})$:
\begin{romannum}
\item  If $\text{Real}(\lambda)<-\half$ for every eigenvalue $\lambda$ of $g A$, then
\[
\Sigma_n = n^{-1} \Sigma_{\theta}  +  O( n^{-1-\delta})
\]
where $\delta=\delta(A,\Sigma_\Delta)>0$, and $\Sigma_{\theta} \geq 0$ is the solution to the Lyapunov equation \eqref{e:Lyapg}. Consequently,   $\Expect[\|\tiltheta_n\|^2]$  converges to zero at rate
$1/n$.

\item  Suppose there is an eigenvalue $\lambda$ of $g A$ that satisfies 
\[
-\varrho_0 = \text{Real}(\lambda)> -\half
\]
Let $\nu \neq0$ denote a corresponding left eigenvector, 
and suppose that $\Sigma_\Delta \nu \neq 0$.   Then, $\Expect[ |\nu^\transpose  \tiltheta_n|^2]  $ converges to $0$ at a rate $1/n^{2\varrho_0}$.
Consequently,   $\Expect[\|\tiltheta_n\|^2]$ converges to zero at   rate no faster than $1 / n^{2\varrho_0}$.
\qed 
\end{romannum}
\end{theorem}

\Proposition{t:SAnonlinear} extends the conclusions of \Theorem{t:SAlinearized} to nonlinear SA~\eqref{e:SAintro}. The proof is contained in \Appendix{s:SAnonlinearProof}.


\begin{proposition}
\label{t:SAnonlinear}
Consider the general SA algorithm~\eqref{e:SAintro}.   Suppose {(A1)} -- {(A3)} hold with $A \eqdef \partial_\theta \barf(\theta) |_{\theta = \theta^*}$,  
and that   $\barf$ has the form  
\[
\barf(\theta) = - \theta  +  \barF(\theta) \,,\qquad \theta \in\Re^d
\]
with $\barF$   Lipschitz continuous,  a  strict contraction,  and $C^1$ in a neighborhood of the origin.    
Then, 
\begin{romannum}
\item[(i)]  If $\text{Real}(\lambda)<-\half$ for every eigenvalue $\lambda$ of $ g A$, then
 
\begin{romannum}
\item[(a)]   The  CLT holds for $ \{W_n = \sqrt{n} \tiltheta_n\}$, with   asymptotic covariance   $\Sigma_{\theta} \geq 0$   the solution to the Lyapunov equation \eqref{e:Lyapg}.

\item[(b)] Weak convergence goes beyond bounded and continuous functions:   for any measurable function  $g\colon\Re^d\to\Re$ with at most quadratic growth we have 
\[
\lim_{n\to\infty}  \Expect[ g( W_n) ]   =   \Expect[ g( W_\infty ) ]   \,,\qquad W_\infty\sim N(0,\Sigma_\theta)
\]
In particular,   $\Expect[\|\tiltheta_n\|^2]$  converges to zero at rate
$1/n$, and
\[
\lim_{n\to\infty} n \Sigma_n =   \Sigma_{\theta}   
\]

\end{romannum} 

\item[(ii)]   Suppose there is an eigenvalue $\lambda$ of $g A$ that satisfies 
\[
-\varrho_0 = \text{Real}(\lambda)> -\half
\] 
Let $\nu \neq0$ denote a corresponding left eigenvector, 
and suppose that $\Sigma_\Delta \nu \neq 0$.
 Then, $\Expect[ |\nu^\transpose  \tiltheta_n|^2]  $ converges to $0$ at a rate $1/n^{2\varrho_0}$.
Consequently,   $\Expect[\|\tiltheta_n\|^2]$ converges to zero at   rate no faster than $1 / n^{2\varrho_0}$.
\end{romannum} 
\end{proposition}


It seems likely that the finite-$n$ bound in \Theorem{t:SAlinearized} also holds for the nonlinear SA algorithm.   We believe that coupling techniques of \cite{sze97} is one way to establish such results for Q-learning.   More importantly, even though the finite-$n$ result remains a conjecture, we have already highlighted how the CLT is often  predictive of finite-$n$ performance.

%



\smallskip

\noindent
\textbf{Organization:}   \Section{sec:MDP} sets notation and provides background on MDPs,  and  \Section{sec:q} contains background on Q-learning (along with new interpretations on the convergence rate of these algorithms).   \Section{sec:RelQ} is devoted to the new relative Q-learning algorithm,  and 
\Section{s:disc} contains directions for future research and conclusions.



\section{Markov Decision Processes Formulation}
\label{sec:MDP}

Consider a Markov Decision Processes (MDP) model with state space $\state$, action space   $\U$,   cost function $c\colon\state\times\U\to\Re$, and discount factor $\gamma\in(0,1)$.  It is assumed throughout that the state and action spaces are finite: denote $\nd = |\state|$ and  $\ninp = |\U|$.  
In the following, the terms `action', `control', and `input' are used interchangeably.

Along with the state-action process $(\bfmX,\bfmU)$ is an i.i.d.\ sequence $\bfmI=\{I_1,I_2,\dots\}$ used to model a randomized policy.   
It is assumed without loss of generality that each $I_n$ takes values in a finite set. 
An input sequence $\bfmU$ is called
\defn{non-anticipative}  if 
\[
U_n = z_n(X_0, U_0,I_1\dots,U_{n-1}, X_n,I_n) \,, \qquad n \geq 0
\]
where $\{z_n\}$ is a sequence of   functions. 



Under the assumption that the state and action spaces are finite, it follows that there are a finite number of deterministic stationary policies $\{ \phi^{(i)} : 1\le i\le \nphi\}$, where each $\phi^{(i)}: \state \to \U$, and $\nphi \le (\ninp)^\nd$. 
A randomized stationary policy is defined by a probability mass function (pmf) $\upmf$ on the $\{ 1,2,\ldots, \nphi\} \times \state$, such that  
\begin{equation}
\begin{aligned}
U_n & =  \sum_{k=1}^{\nphi}  \iota_n(k)  \phi^{(k)}(X_n) 
\end{aligned} 
\label{e:RandStatPolicy}
\end{equation}
with $    \upmf(k,x)  =
\Prob\{\iota_n(k) = 1 \mid X_0\,, \ldots \,, X_{n-1}, X_n = x\} $   for each $n \geq 0$, $ 1\le k \le \nphi$, and $x\in \state$.
It is assumed that  $\iota_n$ is a fixed function of $(I_n,  X_n)$ for each $n$.  



For each $u \in \U$, the controlled transition matrix $P_u$ acts on functions $V\colon\state\to\Re$ via
\begin{equation*}
\begin{aligned}
P_u V \, (x)  &\eqdef \sum_{x'\in\state} P_u(x,x') V (x')
\\
&
= \Expect[V(X_{n+1}) \! \mid \! X_n \!=\! x \,, U_n \!=\! u \,;   X_k, I_k, U_k \! : \! k <  n ]
\end{aligned} 
\end{equation*}
where the second equality holds for any non-anticipative input sequence $\bfmU$.
For any deterministic stationary policy $\phi$,  let $S_\phi$ denote the substitution operator, defined for any function $q\colon \state\times\U\to\Re$ by 
\[
S_\phi q\, (x) \eqdef q(x,\phi(x))
\]
If the policy $\phi$ is randomized, of the form \eqref{e:RandStatPolicy},  we then  define
\[
S_\phi q\, (x) =  \sum_k \upmf(k) q(x,\phi^{(k)} (x)) 
\]  
With  $P$ viewed as  a single matrix with $\nd\cdot\ninp$ rows and $\nd$ columns, and $S_\phi$ viewed as a matrix with $\nd$ rows and $\nd\cdot\ninp$ columns,  the following interpretations hold:
\begin{lemma}
\label{t:PSphi}
Suppose that $\bfmU$ is defined using a   stationary policy $\phi$ (possibly randomized).  Then, both $\bfmX$ and the pair process $(\bfmX,\bfmU)$ are Markovian, and
\begin{romannum}
\item
$P_\phi \eqdef S_\phi P$ is the transition matrix for   $\bfmX$.

\item  $P S_\phi$ is the transition matrix for   $(\bfmX,\bfmU)$.
\qed
\end{romannum}
\end{lemma}

\subsection{Q-function and the Bellman Equation}

For any (possibly randomized) stationary policy $\phi$, we consider two value functions
\begin{subequations}
\begin{align}
\begin{split}
V_\phi(x) 
& \eqdef \sum_{n = 0}^{\infty} (\gamma P_\phi)^n S_\phi  c \,  (x) 
\label{e:hphi}
\end{split}
\\
\begin{split}
Q_\phi(x,u) 
& \eqdef \sum_{n = 0}^{\infty} (\gamma  P S_\phi)^n c \, (x, u) 
\label{e:qphi}
\end{split}
\end{align}
\end{subequations}
which are related via
\begin{equation}
Q_\phi(x,u)  = c(x,u) + \gamma P_u V_\phi \, (x) 
\label{e:VQ}
\end{equation}

The function $V_\phi: \state \to \Re$ in \eqref{e:hphi} is the value function that corresponds to the policy $\phi$ (with the corresponding transition probability matrix $P_\phi$), and cost function $S_\phi c$, that appears in TD-learning algorithms \cite{tsiroy97a,sutbar18}. The function $Q_{\phi}: \state \times \U \to \Re$  is the fixed-policy Q-function considered in the SARSA algorithm \cite{meysur11,rumnir94,sze10}.

The minimal (optimal) value function is denoted
\[
V^*(x) \eqdef \min_{\phi} \, V_\phi(x)  \,, \quad x \in \state
\]
It is known that this is the unique solution to the following Bellman equation:
\begin{equation}
V^*(x) = \min_u \bigg\{ c(x,u) +  \gamma \sum_{x'\in\state}  P_u(x,x')  V^*(x')     \bigg\}     
\label{e:DCOE}
\end{equation}
Any minimizer defines a deterministic stationary policy $\phi^*\colon\state\to\U$ that is optimal over all input sequences \cite{ber12a}:
\begin{equation}
\phi^*(x) \in \argmin_u \bigg\{ c(x,u) +  \gamma \sum_{x'\in\state}  P_u(x,x')  V^*(x')     \bigg\}      
\label{e:opt_policy_h}
\end{equation}

The Q-function associated with $V^*$ is given by \eqref{e:VQ} with $\phi\!=\!\phi^*$,  which is precisely the term within the brackets in \eqref{e:DCOE}:
\[
Q^*(x,u) \eqdef  c(x,u) + \gamma P_u V^*\, (x)    
\]
The Bellman equation \eqref{e:DCOE} implies a similar fixed point equation for the Q-function:
\begin{equation}
Q^*(x,u) =  c(x,u) + \gamma P_u \uQ^*(x) 
\label{e:DCOE-Q}
\end{equation}
in which $\uQ(x)\eqdef\min_u Q(x,u)$ for any   $Q\colon\state\times\U\to\Re$.  

For any function $q\colon\state\times\U\to\Re$,   let    $\phi^q\colon\state\to\U$  denote  an associated policy that satisfies
\begin{equation}
\phi^q(x) \in  \argmin_uq(x,u) 
\label{e:phi_q}
\end{equation}
It is assumed to be specified \emph{uniquely} as follows: 
\begin{equation}
\begin{aligned}
\phi^q & \eqdef \phi^{(\kappa)}\,\,\text{such that}\,\, 
\\
\kappa & = \min \{ i : \phi^{(i)} (x) \! \in \! \argmin_uq(x,u), \, \text{for all }x \in \state  \}
\end{aligned}
\label{e:phi_q_def}
\end{equation}
Using the above notations, and the definitions in \Lemma{t:PSphi}, the fixed point equation~\eqref{e:DCOE-Q} can be rewritten as
\begin{equation}
Q^* (x,u) \!=\!  c + \gamma  P S_{\phi^*} Q^*(x,u), \,\, \text{where $\phi^* \!=\!  \phi^q $,   $q\!=\!{Q^*}$}
\label{e:DCOE-Qb}
\end{equation}

\section{Q-learning}
\label{sec:q}


The goal in Q-learning is to approximately solve the fixed point equation \eqref{e:DCOE-Q}, \emph{without} assuming knowledge of the controlled transition matrix. We restrict the discussion here to the case of linear parameterization for the Q-function: $Q^\theta(x,u)=   \theta^\transpose  \psi (x,u) $, where $\theta\in\Re^d$ denotes the parameter vector, and $\psi\colon\state\times\U\to\Re^d$ denotes the vector of basis functions. 


For a given parameter vector $\theta \in \Re^d$, let $\clB^\theta: \state \times \U\to \Re$ denote the corresponding Bellman error:
\begin{equation}
\clB^\theta (x,u) \eqdef c(x , u)  + \gamma P_{u} \uQ^{\theta}(x) - Q^{\theta} (x, u)
\label{e:BE_theta}
\end{equation}
A \defn{Galerkin approach} to approximating the optimal Q-function $Q^*$ is  formulated as follows: Obtain a non-anticipative input sequence $\bfmU$ (using a randomized stationary policy $\phi$), and a 
$d$-dimensional stationary stochastic process $\bfelig$ that is adapted to $(\bfmX,\bfmU)$.
The \emph{Galerkin relaxation} of the fixed point equation \eqref{e:DCOE-Q} is the root finding problem:  Find $\theta^*$ such that,
\begin{equation}
\begin{aligned}
\barf_i(\theta^*) & \eqdef \Expect\Big[ \tilBE^{\theta^*}_{n+1}  \elig_n(i)\Big] =0 \,, \qquad 1 \leq i \leq d
\end{aligned}
\label{e:eligQalg}
\end{equation}
where, for each $\theta \in \Re^d$, $\tilBE^\theta_{n+1} $ is the ``temporal difference''
\begin{equation}
\tilBE^\theta_{n+1} \eqdef  c(X_n,U_n)   + \gamma     \uQ^{\theta} (X_{n+1})  - Q^{\theta}(X_n,U_n) \,,
\label{e:BE_theta_n}
\end{equation}
and the expectation in \eqref{e:eligQalg} is with respect to the steady state distribution of  $(\bfmX,\bfmU,\bfelig)$. Equation \eqref{e:eligQalg} is often called the \emph{projected Bellman equation.} It is   a special case of the general root-finding problem that is the focus of SA algorithms.

The following Q($0$)~algorithm is the SA algorithm \eqref{e:SAa}, applied to estimate $\theta^*$ that solves \eqref{e:eligQalg}:
For initialization $\theta_0\in\Re^d  $, define the sequence of estimates  recursively:  
\begin{equation}
\theta_{n+1} = \theta_n + \alpha_{n+1} \elig_n  \tilBE^{\theta_n}_{n+1}  \,, \qquad   \elig_n  = \psi(X_{n}, U_{n}) 
\label{e:Qlambda}
\end{equation}
The   choice  for the sequence of eligibility vectors $\{\elig_n  \}$   in \eqref{e:Qlambda} is inspired by the TD($0$) algorithm \cite{sut88,tsiroy97a}.


For a  sequence of $d \times d$ matrices $\bfmG=\{G_n\}$, the \defn{matrix-gain {\rm Q(0)} algorithm} is described as follows:
For initialization $\theta_0 \in\Re^d  $, the sequence of estimates are defined recursively:  
\begin{align}
\theta_{n+1} = \theta_n + \alpha_{n+1} G_{n+1} \psi(X_{n}, U_{n})   \tilBE^{\theta_n}_{n+1} 
\label{e:HQlambda}
\end{align}
A common choice is
\begin{equation}
G_n = \left(  \frac{1}{n} \sum_{k=1}^n \psi ( X_k , U_k ) \psi^\transpose ( X_k , U_k  )  \right)^{-1}
\label{e:WatGainGen}
\end{equation}

The success of these algorithms has been demonstrated in a few restricted settings, such as optimal stopping   \cite{tsiroy99,choroy06,yuber13},   deterministic optimal control \cite{mehmey09a},   and the tabular setting discussed next.  

\subsection{Tabular Q-learning}
\label{sec:q:watkins}

The basic Q-learning algorithm of Watkins \cite{watday92a,wat89} (also known as \defn{``tabular'' Q-learning}) is a particular instance of the Galerkin approach \eqref{e:Qlambda}.  The basis functions are taken to be indicator functions:  
\begin{equation}
\psi_i(x,u) = \ind\{(x,u)=(x^{i}, u^{i}) \}\,,\quad 1\le i\le d
\label{e:WatkinsBasis}
\end{equation}
where $\{(x^{k},u^{k}) : 1\le k\le d\}$ is an enumeration of all state-input pairs, with $d=\nd \cdot \ninp$.  The goal of this approach is to \emph{exactly} compute the function $Q^*$. Substituting $\elig_n \equiv \psi(X_n , U_n)$ with $\psi$ defined in \eqref{e:WatkinsBasis}, the objective \eqref{e:eligQalg} can be rewritten as follows: Find $\theta^* \in \Re^d$ such that, for each $1\le i\le d$,
\begin{align}
0 & = \Expect\bigl[ \tilBE^{\theta^*}_{n+1}   \psi_i(X_n, U_n) \bigr] 
\label{e:eligQalg_Wat_a}
\\
& = \Big[c(x^i , u^i)   + \gamma \Expect\bigl[ \uQ^{\theta^*} (X_{n+1}) | X_n = x^i \,, U_n = u^i \bigr] 
\label{e:eligQalg_Wat_b}
\\
& \hspace{1.4in}- Q^{\theta^*}(x^i, u^i)  \Big ]   \pie(x^i, u^i)
\nonumber
\end{align}
where the expectation in \eqref{e:eligQalg_Wat_a} is in steady state, and $\pie$ in \eqref{e:eligQalg_Wat_b} denotes the invariant pmf of the Markov chain $(\bfmX,\bfmU)$. 
The conditional expectation in \eqref{e:eligQalg_Wat_b} is
\[
\Expect\bigl[ \uQ^{\theta^*} (X_{n+1}) | X_n = x^i \,, U_n = u^i \bigr] =  P_{u^i} \uQ^{\theta^*}(x^i)
\]
Consequently, \eqref{e:eligQalg_Wat_b} can be rewritten as
\begin{equation}
\begin{aligned}
0 & = \clB^{\theta^*}\! (x^i , u^i) \pie(x^i, u^i)
\end{aligned}
\label{e:eligQalg_Wat_sim}
\end{equation}
If $\pie(x^i, u^i) > 0$ for each $1 \leq i \leq d$, then the function $Q^{\theta^*}$ that solves \eqref{e:eligQalg_Wat_sim} is identical to the optimal Q-function in \eqref{e:DCOE-Q}.



There are three  flavors of Watkins' Q-learning that are common in the literature. We discuss each of them below.

\textbf{Asynchronous Q-learning:}
The SA algorithm applied to solve \eqref{e:eligQalg_Wat_a} coincides with the most basic version of Watkins' Q-learning algorithm:
For initialization $\theta_0 \in\Re^d$, define the sequence of estimates $\{\theta_n: n \geq 0 \}$ recursively:
\begin{equation}
\theta_{n+1} = \theta_n + \alpha_{n+1} \tilBE^{\theta_n}_{n+1}   \psi(X_n,U_n)
\label{e:Watkin}
\end{equation}

Algorithm \eqref{e:Watkin} coincides with the Q($0$) algorithm \eqref{e:Qlambda}, with $\psi$ defined in \eqref{e:WatkinsBasis}.
Based on this choice of basis functions, a single entry of $\theta$ is updated at each iteration, corresponding to the state-input pair $(X_n,U_n)$ observed  (hence the term ``asynchronous'').    The parameter   $\theta$ can be identified with the function $Q^\theta$ in this tabular setting.  This equivalence justifies a  slight abuse of notation:  replace $Q^\theta$   by $Q$ and set
  $\tilBE^{Q}_{n+1} \!=\! \tilBE^\theta_{n+1}$ (defined in \eqref{e:BE_theta_n}), resulting in a more  familiar form of \eqref{e:Watkin}:
\begin{equation}
\begin{aligned}
Q^{{n+1}}(X_n, U_n) = Q^{{n}} & (X_n, U_n) + \alpha_{n+1} \tilBE^{Q^n}_{n+1}
\end{aligned}
\label{e:Watkin_fam}
\end{equation}
and $Q^{{n+1}}(x,u) = Q^{{n}} (x,u) $ if $(x,u)\neq (X_n, U_n) $.

With $\alpha_n=1/n$, the \emph{ODE approximation} of \eqref{e:Watkin} takes the form (see \cite{bormey00a} for details):
\begin{equation}
\ddt q_t(x,u) = \pie(x,u)  \Big[  c(x,u)   + \gamma   P_u\uq_t\, (x)   - q_t(x,u)  \Big]   
\label{e:QODEW}
\end{equation}
in which $\uq_t(x)=\min_u q_t(x,u)$ as defined below \eqref{e:DCOE-Q}. 
We recall in \Section{s:wat_conv_ana} conditions under which this ODE is stable, and explain why we cannot expect a finite asymptotic covariance in typical settings.   
\smallbreak

A second and perhaps more popular ``Q-learning flavor'' is defined using a particular   ``state-action dependent'' step-size  \cite{sze97,eveman03,devbusmey19}.   For each  $(x,u)$,  denote $\alpha_n(x,u) =0$ if the pair $(x,u)$ has not been visited up until time $n-1$. Otherwise,    
\begin{equation}
\!\!
\alpha_n (x, u ) \! = \frac{1}{n(x, u)}\,, \,\,\, n(x, u) = \! \sum_{j = 0}^{n-1} \ind\{X_j =x, U_j =u\}
\label{e:wat_2_ss}
\end{equation}
The ODE approximation of \eqref{e:Watkin_fam}   simplifies with \eqref{e:wat_2_ss}:
\begin{equation}
\ddt q_t(x,u) =  c(x,u)   + \gamma   P_u\uq_t\, (x)   - q_t(x,u)
\label{e:QODEW_Gain}
\end{equation}

The asynchronous variant of Watkins' Q-learning algorithm \eqref{e:Watkin} with step-size \eqref{e:wat_2_ss} can be viewed as an instance of $G$-Q($0$) algorithm defined in \eqref{e:HQlambda}, with the matrix gain sequence \eqref{e:WatGainGen}, and step-size $\alpha_n = 1/n$. On substituting the Watkins' basis defined in \eqref{e:WatkinsBasis}, we find that  this matrix is diagonal: $G_n = \haPi_n^{-1}$, where
\begin{equation*}
\begin{aligned}
\haPi_n  (i, \, i) = \frac{1}{n} \sum_{k=1}^n \ind\{X_k = x^{i}, U_k =u^{i}\} \,, \qquad 1\le i\le d
\end{aligned}
\label{e:WatGainWat}
\end{equation*}
By the Law of Large Numbers, we have 
\begin{equation}
\lim_{n \to \infty} G_n = \lim_{n \to \infty} \haPi_n^{-1} = \diagpie^{-1}
\label{e:WatGainWatLim}
\end{equation}
where $\diagpie$ is a diagonal matrix with entries $\diagpie(i, i) = \pie(x^i,\, u^i)$. It is easy to see why the ODE approximation \eqref{e:QODEW} simplifies to \eqref{e:QODEW_Gain} with this matrix gain.



\smallbreak

\textbf{Synchronous Q-learning:} 
In this final flavor, each entry of   the Q-function approximation is updated in each iteration.
It is popular in the literature because the analysis is greatly simplified.

The algorithm requires a model that provides the next state of the Markov chain, conditioned on any given current state-action pair:  let $\{X_n^i : n\ge 1, \,  1\le i\le d\}$ denote a collection of mutually independent random variables taking values in $\state$.  Assume moreover that for each $i$,  the sequence $ \{X_n^i : n\ge 1\}$  is i.i.d.\ with common distribution $P_{u^i}(x^i,  \varble)$.
The   \defn{synchronous Q-learning} algorithm is then obtained as follows: 
For initialization $\theta_0 \in\Re^d  $, define the sequence of estimates $\{\theta_n: n \geq 0\}$ recursively:
\begin{equation}
\begin{aligned}
\theta_{n+1} & = \theta_{n}  + \alpha_{n+1}  \sum_{i=1}^{d} \bigl [  c(x^i, u^i)   + \gamma  \uQ^{\theta_{n}} (X_{n+1}^i)  
\\
& \hspace{1.4in} - Q^{\theta_n}(x^i, u^i)  \bigr ] \psi(x^i \,, u^i) 
\label{e:synWat}
\end{aligned}
\end{equation}
Once again, based on the choice of basis functions \eqref{e:WatkinsBasis}, and observing that $\theta$ is identified with the estimate $Q^\theta$, an equivalent form of the update rule   \eqref{e:synWat} is
\begin{equation}
\begin{aligned}
Q^{{n + 1}}& (x^i, u^i) = Q^{{n}} (x^i, u^i) + \alpha_{n+1} \bigl [  c(x^i, u^i)   
\\
& + \gamma  \uQ^{{n}} (X_{n+1}^i)  - Q^{n}(x^i, u^i)  \bigr ] ,\ \  1 \leq i \leq d
\label{e:synWatfam}
\end{aligned}
\end{equation}
Using the step-size $\alpha_n=1/n$ we   obtain the simple ODE approximation \eqref{e:QODEW_Gain}.

%


\subsection{Convergence and Rate of Convergence}
\label{s:wat_conv_ana}

Convergence of the tabular Q-learning algorithms can be established under the following assumptions: 
\assume{(Q1)}  The input $\bfmU$ is defined by a randomized stationary policy of the form 
\eqref{e:RandStatPolicy}.    The  joint process $(\bfmX,\bfmU)$ is an irreducible Markov chain. That is, it has a unique invariant pmf $\pie$ satisfying  $\pie(x,u)>0$ for each $x,u$.

\assume{(Q2)}  The optimal policy $\phi^*$ is unique.    \qed


Both ODEs \eqref{e:QODEW} and \eqref{e:QODEW_Gain} are stable under assumption~(Q1)~\cite{borsou97a}, which then (based on the results of \cite{bormey00a}) implies that $\bftheta$ converges to $Q^*$ a.s..    To obtain rates of convergence requires an examination of the linearization of the ODEs at their equilibrium.  

Linearization is justified under 
Assumption~{(Q2)}, which implies the existence of   $\epsy>0$ such that  
\begin{equation}
\phi^*(x) = \argmin_{u\in\U} Q^\theta(x,u) \,, \quad \text{if} \,\,\,\, \|Q^\theta - Q^*\|<\epsy\,
\end{equation}
\begin{lemma}
\label{t:WatkinsLin}
Under Assumptions~(Q1) and (Q2) the following approximations hold 
\begin{romannum}
\item  When $\| q_t - Q^*\| <\epsy$, the ODE \eqref{e:QODEW}  reduces to 
\[
\ddt q_t =   - \diagpie   [I - \gamma    PS_{\phi^*} ]   q_t  -    b 
\]
where $\diagpie$ is defined below \eqref{e:WatGainWatLim},  and  $b(x,u) = -\pie(x,u) c(x,u)$, expressed as a $d\times 1$ column vector.

\item  When $\| q_t - Q^*\| <\epsy$, the ODE \eqref{e:QODEW_Gain}  reduces to 
\[
\ddt q_t =   -  [I - \gamma    PS_{\phi^*} ]   q_t  -    b
\]
where    $b(x,u) = - c(x,u)$.
\null 
\end{romannum}
\end{lemma}
The proof is contained in \Appendix{sec:Q_appendix}.


The definition of the linearization matrix $A$ in \eqref{e:Ab} is extended to non-linear functions as follows
\cite{devmey17a,devbusmey20}:
\[
A = \partial_\theta \barf(\theta) \big|_{\theta = \theta^*}
\]
The crucial take-away from \Lemma{t:WatkinsLin} is the linearization matrix that corresponds to each tabular Q-learning algorithms:   
\begin{subequations}
\begin{align}
\begin{split}
A &=
- \diagpie [I - \gamma P S_{\phi^*} ]   \quad \text{in case (i) of \Lemma{t:WatkinsLin}}
\label{e:A_Q}
\end{split}
\\[.4cm]
\begin{split}
A &=  - [I - \gamma P S_{\phi^*} ]   \quad \,\,\,\, \hspace{0.02in} \text{in case (ii) of \Lemma{t:WatkinsLin}}
\label{e:A_Qa}
\end{split}
\end{align}
\label{e:A_QandQa}
\end{subequations}
Since $\gamma < 1$, and $PS_{\phi^*}$ is a transition matrix of an irreducible Markov chain (see \Lemma{t:PSphi}), it follows that both  matrices are Hurwitz.

We consider next conditions under which the asymptotic covariance for Q-learning is \emph{not} finite.   
The noise covariance matrix $\Sigma_\Delta$ defined in \eqref{e:SigmaDelta}  is diagonal in all three flavors of Q-learning discussed in \Section{sec:q:watkins}.   
For the asynchronous Q-learning algorithm \eqref{e:Watkin_fam} with step-size \eqref{e:wat_2_ss}, 
or the synchronous Q-learning algorithm \eqref{e:synWatfam}, the diagonal elements of $\Sigma_\Delta$  are given by   $ \Sigma^{s(i, i)}_\Delta = $
\begin{align}
 &\! \gamma^2 \Expect \Big [ \! \big( \uQ^{*} \! (X_{n+1}) \! - \! P_{u^i} \uQ^* \! (x_i) \big)^2 \Big | X_n \! = \! x^i , U_n \! = \! u^i  \Big]
\nonumber
\\
\!=\! & \gamma^2  \Expect \Big [ \! \big( V^{*} \! (X_{n+1}) \! - \! P_{u^i} \! V^* \! (x_i) \big)^2 \Big | ( \! X_n = x^i , \! U_n = u^i \!) \! \Big]
\label{e:SigmaDeltaQ}
\end{align}
The noise covariance for asynchronous Q-learning with step-size $\alpha_n \! = \! 1/n$ is $\Sigma^a_\Delta \! = \! \diagpie \Sigma^s_\Delta \diagpie$, with $\diagpie$ defined below \eqref{e:WatGainWatLim}.

\begin{theorem}
\label{t:Qinfinite}
Suppose that assumptions~{{(Q1)}} and~{{(Q2)}} hold, and $\alpha_n \equiv 1/n$. Then, the sequence of parameters $\{ \theta_n \}$ obtained using the asynchronous Q-learning algorithm \eqref{e:Watkin} converges to $Q^*$ a.s..
Suppose moreover that the conditional variance of $V^* (X_n)$ is positive:
\begin{equation}
\sum_{x,x',u} \pie(x,u) P_u(x,x') [ V^*(x') -  P_u V^*\, (x) ]^2 >0
\label{e:hvar}
\end{equation}
\begin{equation}
\hspace{-1.25in}\text{and } \qquad \qquad (1-\gamma) \max_{x,u} \pie(x,u)  < \half
\label{e:disc_cond}
\end{equation}
Then, 
\begin{romannum}
\item The asymptotic covariance of the algorithm is infinite:
\[
\lim_{n\to\infty}  n \Expect[ \| \theta_{n} - \theta^*\|^2 ] = \infty
\]
\item $\Expect[ \| \theta_{n} - \theta^*\|^2 ]$ converges to zero at a rate no faster than $1 / n^{2(1 - \gamma)}$.\qed
\end{romannum}
\end{theorem}
The inequality \eqref{e:disc_cond} is satisfied whenever the discount factor satisfies $\gamma\ge \half$. 

\Theorem{t:Qinfinite} explains why the Q-learning algorithm can be terribly slow: If the discount factor is close to $1$, which is typical in many applications, using a step-size of the form $\alpha_n = 1/n$ results in a MSE convergence rate that is \emph{much} slower than the optimal rate $1/n$.


Similar conclusions hold for the other flavors of tabular Q-learning, for which the algorithm admits the ODE approximation  \eqref{e:QODEW_Gain}. Based on \Lemma{t:WatkinsLin}, the linearization matrix for these algorithms is defined in \eqref{e:A_Qa}. 
This poses problems when $\gamma>\half$,   but for these algorithms there is a simple remedy:
\begin{theorem}
\label{t:ADgain}
For asynchronous Q-learning with the step-size rule \eqref{e:wat_2_ss}, or
synchronous Q-learning with step-size $\alpha_n=1/n$,  the matrix   shown in \eqref{e:A_Qa} is equal to the linearization matrix $A = \partial_\theta \barf(\theta) \big|_{\theta = \theta^*} $.  It has one eigenvalue $\lambda_1 =-(1-\gamma)$,   and $\Real(\lambda(A)) < -(1-\gamma)$ for every other eigenvalue.  Consequently,  
\begin{romannum} 
\item  Subject to \eqref{e:hvar}, the asymptotic covariance is not finite whenever $\gamma>\half$.
\item   Suppose that the step-sizes are scaled:   use $ \alpha_n(x, u) =    [ (1-\gamma)  n(x, u)   ]^{-1} $ for asynchronous Q-learning,   or $ \alpha_n =    [ (1-\gamma)  n   ]^{-1} $ for synchronous Q-learning. Then, the eigenvalue test passes: for each eigenvalue $\lambda = \lambda(A)$,
\[
\Real (\lambda  )=    - (1-\gamma)^{-1}  \Real\bigl(\lambda( [I - \gamma PS_{\phi^*}]   )\bigr)\le -1
\]
The resulting asymptotic covariance is obtained as a solution to the Lyapunov equation \eqref{e:Lyapg}, with $g = (1-\gamma)^{-1}$, and $\Sigma_\Delta=\Sigma^s_\Delta$ defined in \eqref{e:SigmaDeltaQ}.
\qed  
\end{romannum}
\end{theorem}
The step-size rule  $ \alpha_n =    [ (1-\gamma)  n   ]^{-1} $  is equivalent to  $ \alpha_n =  [ 1+(1-\gamma)  n   ]^{-1} $ that appears in \cite{wai19a}, in the sense that each algorithm will share the same asymptotic covariance.

\smallbreak

\textbf{Overview of  proofs:  }  We begin with \Theorem{t:Qinfinite}. 
The proof of convergence can be found in \cite{watday92a,tsi94a,bormey00a}.
The proof of infinite asymptotic covariance is  based on an application of \Prop{t:SAnonlinear}. A brief overview follows.

To establish the slow convergence rate, an eigenvector for $A$ (defined in \eqref{e:A_Q}) can be constructed with strictly positive entries, and with real part of the corresponding eigenvalue satisfying $\text{Re}(\lambda) \ge -1/2$  (see Appendix~A.2 of \cite{devmey17a}). 
Interpreted as a function $v\colon\state\times\U\to\Co$,  this eigenvector satisfies $ v^\dagger  \Sigma_\Delta v =$
\begin{align}
\label{e:SigmaDeltaQ_condition}
 \gamma^2 \! \sum_{x,u,x'} \! \pie(x,u) |v(x,u)|^2 P_u(x,x') [  V^*(x') \! - \! P_u \! V^* (x) ]^2
\end{align}
where $\Sigma_\Delta$ is the noise covariance matrix (recall \eqref{e:SigmaDeltaQ}),
and $v^\dagger$ denotes complex-conjugate transpose.
Assumption \eqref{e:hvar} ensures that the right hand side of \eqref{e:SigmaDeltaQ_condition} is strictly positive, as required in part~(ii) of \Proposition{t:SAnonlinear}.



\smallbreak

\Theorem{t:ADgain} is based on the simple structure of the eigenvalues of the linearization matrix  $A  =  - [I - \gamma P S_{\phi^*} ] $ defined in \eqref{e:A_Qa}.   Because $P S_{\phi^*}  $ is the transition matrix for an irreducible Markov chain, it follows that all of its eigenvalues are in the closed unit disk in the complex plane, with a single eigenvalue at $\lambda=1$.   Consequently,  $A$ has a single eigenvalue at $\lambda=-(1-\gamma)$,
and $\Real(\lambda(A)) <  -(1-\gamma)$ for all other eigenvalues.  An application of \Proposition{t:SAnonlinear} then implies both (i) and (ii) of the theorem.
\qed

Theorems~\ref{t:Qinfinite} and \ref{t:ADgain} motivate the introduction of new algorithms whose performance does not degrade with large $\gamma$.

\section{Relative Q-learning}
\label{sec:RelQ}

%


The following \defn{relative Bellman equation} was inspired by the decomposition \eqref{e:tilQ}:   
\begin{equation}
H^*(x,u) = c(x,u) + \gamma P_u \uH^*(x)   - \delta \langle \mu \,, H^* \rangle
\label{e:DCOE-H}
\end{equation}
where $\delta >0$ is a positive scalar, $\mu: \state \times \U \to [0,1]$ is a pmf (both design choices), and 
\[
\langle \mu \,, H^* \rangle = \sum_{x  \,, u } \mu (x, u) H^*(x, u)
\]
For example, we may choose $\mu(x,u) = \ind\{ x = x^\bullet, u = u^\bullet \}$ for some fixed $(x^\bullet\,, u^\bullet) \in \state \times \U$, so that $\langle \mu , H \rangle = H(x^\bullet , u^\bullet)$ for any $H$. 

With $\gamma =1$, the fixed point equation \eqref{e:DCOE-H} is very similar to the fixed point equation that appears in the average cost Q-learning formulation of \cite{aboberbor01},  though the motivations are different:  the prior work is devoted to Q-learning algorithm for the average cost criterion, while the present paper concerns reliable algorithms in the discounted cost setting.

 

Define  $\tilH^*(x,u) \eqdef Q^*(x,u) -  \langle \mu \,, Q^* \rangle $, which by  \eqref{e:tilQ} can be expressed
\[
\tilH^*(x,u) = \tilQ^*(x,u) -  \langle \mu \,, \tilQ^* \rangle 
\]
It follows that $\tilH^*$ is uniformly bounded in $\gamma$, $x$, and $u$ \cite{put14,ber12a}.
The relationship (i) in \Proposition{t:Hstar} is immediate from the definitions.  
Part (ii) implies that $H^*$ is uniformly bounded over $\gamma\in [0,1)$.
Observe that \eqref{e:HQ} implies that   $Q^*$ can be recovered from $H^*$ and $\mu$. 
 
\begin{proposition}
\label{t:Hstar}
Under (Q1)--(Q2),
the solution $H^*$ to \eqref{e:DCOE-H} is unique, and satisfies:
\begin{romannum}
\item  $\displaystyle H^*(x, u) = Q^*(x,u)  - k$, with
\begin{equation}
k = \frac{\delta}{ 1+\delta -\gamma}  \langle \mu \,, Q^* \rangle = \frac{\delta}{ 1 -\gamma}  \langle \mu \,, H^* \rangle 
\label{e:HQ}
\end{equation}
\smallskip
\item  $\displaystyle H^*(x, u) = \tilH^*(x,u)   + \eta^*/\delta + o(1)$,    where $o(1)\to 0$ as $\gamma\uparrow 1$.
\end{romannum}  
\qed
\end{proposition}  
\begin{proof}
The proof of (i) follows from \eqref{e:DCOE-H} and \eqref{e:DCOE-Q}. 
This further implies
\[
\begin{aligned} 
H^*(x, u)& = Q^*(x,u)  - \Bigl( 1 -  \frac{1 -\gamma}{ 1+\delta -\gamma}  \Bigr)  \langle \mu \,, Q^* \rangle
\\
& = \tilH^*(x,u)  +  \frac{1}{ 1+\delta -\gamma}  ( 1 -\gamma)  \langle \mu \,, Q^* \rangle
\end{aligned} 
\]
This concludes the proof of (ii), since $ ( 1 -\gamma)  \langle \mu \,, Q^* \rangle \to \eta^*$ as $\gamma\uparrow 1$;  
this well known fact follows from 
\eqref{e:tilQ}   (see also \cite{put14}).   
\end{proof} 

The objective in relative Q-learning is to estimate   $H^*$. Since $Q^*$ and $H^*$ differ only   by a constant, the policy $\phi^*$ defined in \eqref{e:DCOE-Qb} satisfies $\phi^* = \phi^q$, with $q = H^*$ (see \eqref{e:phi_q}).
It is therefore irrelevant whether we estimate $Q^*$ or $H^*$, if we are ultimately interested only in the optimal policy.

We conjecture that estimating $H^*$ results in finite-$n$ error bounds of the form \eqref{e:poly_gamma_discount}, which is uniformly bounded for all $\gamma < 1$ (in sharp contrast to   finite-$n$ bounds for estimating $Q^*$---recall \eqref{e:poly_discount}).    
We establish here that  the asymptotic covariance is uniformly bounded in $\gamma$ under the right choices for $\delta$ and the step-size.


\subsection{Relative Q-learning Algorithm}
\label{s:RelQAlgo}

Consider a linear parameterization for the relative Q-function: $H^\theta(x,u)=   \theta^\transpose  \psi (x,u) $, where $\theta\in\Re^d$ denotes the parameter vector, and $\psi\colon\state\times\U\to\Re^d$ denotes the vector of basis functions. We restrict the discussion here to the tabular case, where the basis functions $\{\psi_i: 1 \leq i \leq d\}$ are the indicator functions defined in \eqref{e:WatkinsBasis}.


The goal in \emph{tabular relative Q-learning} is to find $\theta^*$ such that
\begin{equation}
\begin{aligned}
\barf (\theta^*) & \eqdef \Expect\bigl[ \bigl\{  c(X_n,U_n)   +   \uH^{\theta^*} (X_{n+1})  - \delta \langle \mu \,, H^{\theta^*} \rangle 
\\
&\hspace{0.4in} - H^{\theta^*}(X_n,U_n)   \bigr\}   \psi(X_n, U_n)\bigr] =0
\label{e:eligHalg}
\end{aligned}
\end{equation}
where $\bfmU$ is a non-anticipative input sequence (obtained using a randomized stationary policy $\phi$), $\uH^\theta(x) = \min_u H^\theta (x,u)$, and the expectation is with respect to the steady state distribution of the Markov chain $(\bfmX,\bfmU)$. With the basis functions chosen to be indicator functions \eqref{e:WatkinsBasis}, interpretations similar to \eqref{e:eligQalg_Wat_a}--\eqref{e:eligQalg_Wat_sim} hold, and the objective \eqref{e:eligHalg} can be rewritten as: For each $1\le i\le d$,
\begin{equation}
\begin{aligned}
\barf_i (\theta^*) & = \Big[ c(x^i , u^i) + \gamma P_{u^i} \uH^{\theta^*}(x^i)  
\\
& \hspace{0.3in}  - \delta \langle \mu , H^{\theta^*} \rangle  - H^{\theta^*}(x^i, u^i) \Big ]   \pie(x^i, u^i) = 0
\end{aligned}
\label{e:eligRelQalg_Wat_sim}
\end{equation}
where $\pie$ denotes the invariant pmf of $(\bfmX,\bfmU)$.

We once again assume (Q1) and (Q2) of \Section{s:wat_conv_ana} throughout.
Under~(Q1), it is easy to see that $H^{\theta^*}$ that solves \eqref{e:eligRelQalg_Wat_sim} is identical to the optimal relative Q-function in \eqref{e:DCOE-H}. Assumption~(Q2) implies existence of $\epsy>0$ such that
\begin{equation}
\phi^*(x) = \argmin_{u\in\U} H^\theta(x,u) \,, \qquad \|H^\theta - H^*\|<\epsy\,
\label{e:phi_star_H}
\end{equation}


As in \Section{sec:q:watkins}, there are many flavors of relative Q-learning algorithm that are possible. We restrict our discussion here to the \emph{asynchronous relative Q-learning} algorithm, which requires access to a single sample path of the Markov chain $(\bfmX,\bfmU)$. Extension of the results and discussion to other flavors of the algorithm is straightforward.

\textbf{Asynchronous Relative Q-learning:}

The asynchronous algorithm is a direct application of SA to solve \eqref{e:eligHalg}:
For initialization $\theta_0 \in\Re^d$, define the sequence of estimates $\{\theta_n: n \geq 0 \}$ recursively:
\begin{equation}
\begin{aligned}
\theta_{n+1} & = \theta_n + \alpha_{n+1} \Big [  c(X_n,U_n)   + \gamma      \uH^{\theta_n} (X_{n+1})
\\
& \hspace{0.4in}
- \delta \langle \mu \,, H^{\theta_n} \rangle - H^{\theta_n}(X_n,U_n)  \Big ]   \psi(X_n,U_n)
\label{e:DevMey}
\end{aligned}
\end{equation}
Based on the choice of basis functions \eqref{e:WatkinsBasis}, a single entry of $\theta$ is updated at each iteration, corresponding to the state-input pair $(X_n,U_n)$ observed. By identifying $\theta$ with the estimate $H^\theta$, we can rewrite \eqref{e:DevMey} as
\begin{equation}
\begin{aligned}
H^{{n+1}}\! (X_n, U_n) & \! = \! H^{{n}} (X_n, U_n) \! + \! \alpha_{n+1} \bigl [  c(X_n,U_n)  
\\
& \hspace{-0.2in }\! + \! \gamma  \uH^{{n}} (X_{n+1}) \! - \! \delta \langle \mu , H^{n} \rangle \! - \! H^{n}(X_n,U_n)  \bigr ]
\label{e:DevMey_fam}
\end{aligned}
\end{equation}
With $\alpha_n \!=\!1/n$, the \emph{ODE approximation} of \eqref{e:DevMey_fam} takes the form
\begin{equation}
\begin{aligned}
\ddt h_t(x,u) \! = \! \pie(x,u)  \Big[  c(x,u)   & + \gamma   P_u\uh_t\, (x) 
\\
& - \delta \langle \mu \,, h_t \rangle  - h_t(x,u)  \Big]   
\label{e:RelQODE}
\end{aligned}
\end{equation}
in which $\uh_t(x)=\min_u h_t(x,u)$. 
Based on the discussion in \Section{sec:q:watkins}, a ``more efficient'' relative Q-learning flavor is defined using a particular  state-action dependent step-size \eqref{e:wat_2_ss}.
The ODE approximation \eqref{e:RelQODE} simplifies in this case:
\begin{equation}
\ddt h_t(x,u) =  c(x,u)   + \gamma   P_u\uh_t\, (x) - \delta \langle \mu \,, h_t \rangle  - h_t(x,u)
\label{e:RelQODE_Gain}
\end{equation}
Henceforth we restrict  discussion to the relative Q-learning algorithm with a scaling of this specific step-size:   $  \alpha_n(x, u) = g  \cdot \big[n(x, u) \big]^{-1} $ with $g>0$.    We initially assume $g=1$.   

%

\subsection{Stability and Convergence of Relative Q-learning}
\label{sec:conv_rel_q}

Convergence of the algorithm holds under mild conditions:
 
\begin{theorem}[Stability \& Convergence]
\label{e:HbddAndConverges}
Consider the relative Q-learning algorithm \eqref{e:DevMey_fam} with step-size $\alpha_n(x,u)$ satisfying \eqref{e:wat_2_ss}.
 Then,   
$
\lim_{n \to \infty} H^n = H^*$, a.s., for each initial condition.
\qed
\end{theorem}

The proof of the theorem follows from  \cite[Theorems~2.1 and~2.2]{bormey00a},   which tells us that stability of the ODE \eqref{e:RelQODE_Gain} implies firstly that
\[
\sup_n \sup_{x,u} H^n(x,u)  < \infty \qquad a.s.
\]
and then convergence follows from more well known arguments.   Global asymptotic stability of the ODE is established in   \Proposition{t:ODEisstable}.  
A martingale noise assumption is imposed on the SA recursions considered in \cite{bormey00a,bor20a}  (it is argued that the stability result holds for more general Markovian noise).   This extension is not required to prove \Theorem{e:HbddAndConverges}, as we can cast the relative Q-learning algorithm precisely within the setting of  \cite{bormey00a}.


The algorithm in \eqref{e:DevMey_fam}, with step-size rule \eqref{e:wat_2_ss} can be rewritten as:
\begin{equation}
\begin{aligned}
& H^{{n+1}}(x, u) = H^{{n}} (x, u) 
\\
&\hspace{0.2in} + \alpha_{n+1} (x,u) [ \barf(H^n, X_n, U_n ; x,u) 
+ \Delta_{n+1} (x \,, u)    ] 
\vspace{-0.3in}
\label{e:DevMey_fam_mds}
\end{aligned}
\end{equation}
where 
\begin{equation*}
\barf_{H^n} \! (X_n, U_n ; x,u)  \! = \! \big [  \tilT H^n \! (x , u) \! - \! H^{n} \! (x, u) \big]     \ind\{X_n \! =\!  x, U_n \! = \! u\}
\end{equation*}
and for any $H$,
\begin{equation}
\begin{aligned}
\tilT H (x , u) \eqdef c(x,u) + \gamma P_u \uH (x)  - \delta \langle \mu , H \rangle            
\end{aligned} 
\label{e:TilT}
\end{equation}
and $\{\Delta_n\}$ is the noise sequence: $\Delta_{n+1}  (x,u) = $
\begin{equation}
\begin{aligned}
\gamma  \Big(  \uH^{{n}}   (X_{n + 1})   -  P_u \uH^n   (x) \Big)  \ind\{ X_n  =  x , U_n  =  u  \}
\label{e:DeltaH}
\end{aligned}
\end{equation}
The recursion \eqref{e:DevMey_fam_mds} is stochastic approximation with Markovian noise, as assumed in  \cite{bormey00a}.

For the purpose of analysis, it is best to visualize the algorithm   \eqref{e:DevMey_fam_mds} with step-size rule \eqref{e:wat_2_ss} as ``$d$ parallel stochastic approximation algorithms'', one for each state-action pair $(x,u)$. If a particular $(X_n\,, U_n)$ is observed in the $n^{\text{th}}$ iteration, then the corresponding $H$-value is updated, with the rest of the $H$-values left unchanged.

The martingale difference property is expressed as follows:  for each $(x,u) \in \state \times \U$,
\begin{equation}
\Expect [\Delta_{n+1} (x,u) | \clF_n] = 0
\label{e:DeltaMDS}
\end{equation}
where $\clF_n = \sigma (X_m, U_m :  m \leq n)$.   A second assumption of  \cite{bormey00a} also holds:    for some constant $K>0$,
\begin{equation}
\Expect [ \| \Delta_{n+1} (x,u) \|^2 | \clF_n] \leq K (1 + \| H^n \|^2)
\label{e:DeltaMDSbdd}
\end{equation}



%
%
%

\subsection{Convergence Rate of Relative Q-learning}
\label{sec:ConvRateRelQ}

We now analyze the asymptotic covariance of the relative Q-learning algorithm \eqref{e:DevMey_fam} that approximates the ODE \eqref{e:RelQODE_Gain}. Following along the lines of analysis in \Section{s:wat_conv_ana}, the covariance analysis requires two ingredients:  identification of  the noise covariance  $\Sigma_\Delta$ in  \eqref{e:SigmaDelta}, and examination of  the linearization of the ODE \eqref{e:RelQODE_Gain}.  Recall that a finite asymptotic covariance depends on properties of the eigenvalues of the linearization matrix $
A = \partial_\theta \barf(\theta) \big|_{\theta = \theta^*}$. 

As for the first ingredient, it follows from \eqref{e:DeltaH} that  the noise covariance  is a diagonal matrix, with $\Sigma_\Delta^{(  i,i  )}   \! = \!   $
\begin{equation}
\gamma^2  \Expect \Big [ \! \big(   \uH^{*}   (X_{n + 1})   -    P_{u^i} \uH^*   (x^i)   \big)^2   \mid   (X_n,   U_n)   =   (x^i,   u^i ) \!  \Big]
\label{e:SigmaDeltaH}
\end{equation} 
This is identical to the noise covariance in Watkins' algorithm:
\begin{lemma}
\label{t:SigmaDeltaEq}
The noise covariance matrix $\Sigma_\Delta^q$ for the Q-learning algorithm (defined in \eqref{e:SigmaDeltaQ}), and $\Sigma_\Delta^h$ for the relative Q-learning algorithm (defined in \eqref{e:SigmaDeltaH}) are identical.   
\end{lemma}
\begin{proof}
The proof is a direct application of \Prop{t:Hstar}:  with $\kappa_\gamma = \delta  \langle \mu \,, H^* \rangle/( 1 - \gamma )$ we obtain, for each $1 \leq i \leq d$,
\[
\begin{aligned}
\Sigma_\Delta^{q~(i,i)} \! & = \! \gamma^2  \Expect \Big [ \! \big( \uQ^{*} \! (X_{n+1}) \! - \!  P_{u^i} \uQ^* \! (x^i) \big)^2 \! \mid \! X_n \! = \! x^i, U_n \! = \! u^i \! \Big]
\\
& \hspace{-0.4in} = \! \gamma^2  \Expect \Big [ \! \big( \uH^{*} \! (X_{n+1}) \! + \! \kappa_\gamma \! - \!  P_{u^i} \uH^* \! (x^i) \big)^2 \! - \! \kappa_\gamma \! \mid \! X_n \! = \! x^i, U_n \! = \! u^i \! \Big]
\\
& \hspace{-0.4in} = \Sigma_\Delta^{h~(i,i)}
\end{aligned}
\]
\end{proof}
We henceforth denote $\Sigma_\Delta = \Sigma_\Delta^q = \Sigma_\Delta^h$.


We turn next to  the linearization of the ODE \eqref{e:RelQODE_Gain} at its equilibrium:  this is justified under Assumption~{(Q2)}, which implies the existence of $\epsy>0$ such that \eqref{e:phi_star_H} holds.  The following result is a direct analog of \Lemma{t:WatkinsLin} for the relative Q-learning algorithm.
\begin{lemma}
\label{t:DevMeyLin}
Under Assumption~(Q2), when $\| \tilh_t\| <\epsy$, with $\epsy> 0$ used in \eqref{e:phi_star_H}, the ODE \eqref{e:RelQODE_Gain}  simplifies to 
\[
\ddt h_t =   -  [I - \gamma P S_{\phi^*} + \delta \cdot \one \otimes \mu ]   h_t  -    b
\]
where    $b(x,u) = - c(x,u)$.
\qed
\end{lemma}

In \Lemma{t:DevMeyLin}, $\one \in \Re^d$ is viewed as a column vector with each component $\one_i = 1$, $1 \leq i \leq d$, and $\otimes$ denotes the outer product.
The lemma provides a simple expression for the linearization matrix:
\begin{equation}
A = - [I - \gamma P S_{\phi^*} + \delta \cdot \one \otimes \mu ]
\label{e:Amatrix}
\end{equation}


In addition to (Q1) and (Q2), we   impose the following additional assumption for the convergence rate analysis:
\assume{(Q3)} The Markov chain with transition matrix $PS_{\phi^*}$ is uni-chain: the eigenspace corresponding to the eigenvalue $\lambda_1 = 1$ is one-dimensional.  
\qed

Denote  
\begin{equation}
\rho^* = \max \{ \text{Re}( \lambda_i ) :  i \ge 2   \}    
\label{e:rhostar}
\end{equation}
where the maximum is over all eigenvalues of  $PS_{\phi^*}$ except $\lambda_1=1$.  
Under (Q3) we have  $\rho^*<1$, and in fact  $\rho^*<0$ is possible.
Let $\rho$ denote the magnitude of the second largest eigenvalue of $PS_{\phi^*}$:
\begin{equation}
\rho = \max\{  |\lambda_i| : \lambda_i\neq 1 \}
\label{e:rho}
\end{equation}
The scalar $\rho$ is also known as the \emph{mixing rate} of the Markov chain $(\bfmX\,, \bfmU)$, with the input sequence $\bfmU$ defined by $\phi^*$,  and $1-\rho$ is the \defn{spectral gap} of the corresponding transition matrix.    While $\rho^* < 1$  is always true under (Q3),   this does not exclude the possibility that  $\rho = 1$ (i.e., there is no spectral gap). We have an obvious bound:
\begin{lemma}
The quantities $\rho$ and $\rho^*$ defined in \eqref{e:rhostar} and \eqref{e:rho} satisfy
$
\rho \leq \rho^* 
$.
\end{lemma}
The bound is achieved if there is a real and positive eigenvalue satisfying  $\lambda_2 = \rho$.   

The following theorem (which is analogous to \Theorem{t:ADgain} for the Q-learning algorithm) is the main result of this subsection. 
\begin{theorem}
\label{t:RelQADgain}
For the asynchronous relative Q-learning algorithm \eqref{e:DevMey_fam} with step-size rule \eqref{e:wat_2_ss}, the matrix $A$ in \eqref{e:Amatrix} is equal to the linearization matrix $A = \partial_\theta \barf(\theta) \big|_{\theta = \theta^*} $.  If we choose $\delta \geq  \gamma (1 -  \rho^*)$, then each eigenvalue of $A$ satisfies $\Real(\lambda(A))\le -(1 - \gamma \rho^*)$.  Consequently, 
\begin{romannum}
\item  The asymptotic covariance   is infinite if $\gamma \rho^* > \half$,  and also $\nu_2^\dagger \Sigma_\Delta \nu_2>0$,  where $\nu_2$ is an eigenvector  of $PS_{\phi^*}$ with eigenvalue satisfying $ \text{Re}( \lambda_2 ) = \rho^*$. 

\item   Suppose that the step-sizes are scaled:  
\begin{equation}
\alpha_n(x, u) =    [ (1 - \gamma \rho^*) \cdot  n(x, u)   ]^{-1} 
\label{e:g_rel_Q}
\end{equation}
Then, the eigenvalue test passes: each eigenvalue $\lambda(A)$ satisfies
\[
\begin{aligned}
\Real (\lambda (A)) 
& \! = \!  - (1-\gamma \rho^* )^{-1}  \Real\big(\lambda \big(  [I \! - \! \gamma PS_{\phi^*} \! - \! \delta \cdot \one \otimes \mu]  \big) \big) 
\\
& \le -1
\end{aligned}
\]
The asymptotic covariance of the resulting algorithm is obtained as a solution to the Lyapunov equation \eqref{e:Lyapg}, with $g = (1-\gamma \rho^*)^{-1}$, and $\Sigma_\Delta$ defined in \eqref{e:SigmaDeltaH}.
\qed  
\end{romannum}
\end{theorem}




\begin{figure*}[ht]
\includegraphics[width = 1\textwidth]{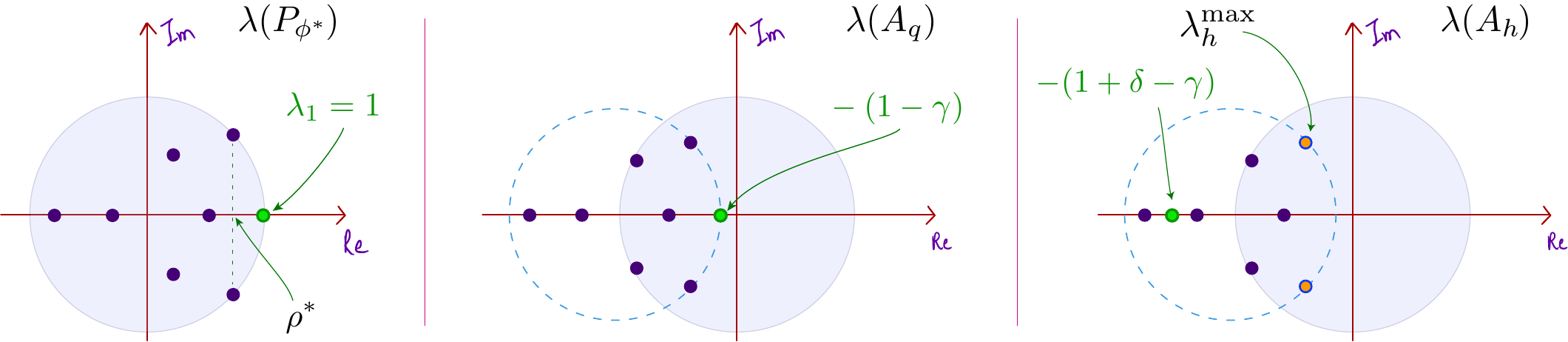}
\caption{Relationship between the eigenvalues of the matrices $PS_{\phi^*}$, $A_q$, and $A$.}
\label{f:Beta}
\end{figure*}


To be clear: \textit{the condition $\rho \! < \! 1$ is not necessary for stability of relative Q-learning,  or uniform boundedness of the asymptotic covariance.}  Consider the example illustrated in  \Fig{f:Beta}. 
The plot of eigenvalues for $PS_{\phi^*}$ shown on the left hand side indicates complex eigenvalues on the unit circle, so that $\rho = 1$. The plots show that $\rho^* < 1$, and therefore, $-(1 - \gamma\rho^*) < - (1 - \gamma)$.
In this case, \Theorem{t:RelQADgain}~(ii) implies that the relative Q-learning algorithm with step-size $\alpha_n = g \cdot [n(x,u)]^{-1}$, $g = - [1 - \gamma \rho^*]^{-1}$ will have finite asymptotic covariance.

We close the section with proof of \Theorem{t:RelQADgain}.

   
{\textbf{Proof of \Theorem{t:RelQADgain}:}} The proof is based on comparing the eigenvalues of the matrix $A$ with the eigenvalues of the linearization matrix that corresponds to the asynchronous Watkins' Q-learning algorithm (recall \Lemma{t:WatkinsLin}~(ii), and Eq.~\eqref{e:A_Qa}):
\begin{equation}
A_q = - [I - \gamma P S_{\phi^*}]
\label{e:AQmatrix}
\end{equation} 

 \begin{lemma}
\label{t:AQ-lemma}
\begin{romannum}
\item 

The matrix $A_q$ is Hurwitz, with all eigenvalues $\lambda$ satisfying $\Real ( \lambda ) \leq -(1-\gamma)$. Furthermore, there exists a single eigenvalue at $\lambda = - (1-\gamma)$, and all other eigenvalues satisfy 
\begin{equation}
\Real ( \lambda (A_q) ) \leq - (1  - \gamma \rho^*)
\label{e:ReLambdaAq}
\end{equation}
where $\rho^* \in [0, 1)$ is defined in \eqref{e:rhostar}.
%
%
%

\item
The vector $\one$ is a right eigenvector of $A$, with eigenvalue
$\lambda_1 = -(1 - \gamma + \delta)$.  
Moreover,  every eigenvalue $\lambda$ of $A$, that is \emph{not} equal to $-(1 - \gamma + \delta)$, is also an eigenvalue of $A_q$, 
with identical left eigenvectors.
\end{romannum}
\end{lemma}

\begin{proof}
The proof of (i) follows from the following observations:
\Theorem{t:ADgain}  combined with assumption (Q3) establishes the upper bound
\[
\Real ( \lambda (A_q) ) \leq - (1 - \gamma)
\] 
The column vector $\one$ is an eigenvector, whose eigenvalue coincides with this bound: 
\[
A_q \one = -(1-\gamma) \cdot \one
\]

We now prove (ii).
The first claim follows from these steps:
\[
\begin{aligned}
A \one 
& = - [I - \gamma P S_{\phi^*} + \delta \cdot \one \otimes \mu ] \one
= - (1 - \gamma + \delta) \cdot \one
\end{aligned}
\]
If $\lambda \neq -(1 - \gamma + \delta) $ is an eigenvalue of $A_q$, with corresponding left eigenvector $\nu$, we have:
\[
\begin{aligned}
\lambda 
\nu^\transpose 
=
\nu^\transpose A  & = - \nu^\transpose [I - \gamma P S_{\phi^*} - \delta \cdot \one \otimes \mu ]
\\
&  \overset{(a)}{=} - \nu^\transpose [I - \gamma P S_{\phi^*}]
\\
& = \nu^\transpose  A_q 
\end{aligned}
\]
where $(a)$ follows from the fact that the left eigenvector $\nu$ is orthogonal to $\one$; this result is formalized in \Lemma{t:orth-lemma} of \Appendix{sec:asym_var_appendix}.
\end{proof}

\Lemma{t:A-eig-values} asserts that \eqref{e:ReLambdaAq} holds for \emph{every eigenvalue} in the relative Q-learning algorithm if $\delta$ is greater than or equal to $1-\rho^*$. Note that $\delta = \gamma$ will always satisfy the condition in \Lemma{t:A-eig-values}. 
The proof is immediate from \Lemma{t:AQ-lemma}.


\begin{lemma}
\label{t:A-eig-values}
Suppose we choose $\delta \geq  \gamma (1- \rho^*)$.
Then, each eigenvalue  of the 
linearization matrix $A$  defined in \eqref{e:Amatrix}
satisfies
\begin{equation}
\Real (\lambda(A)) \leq - (1  - \gamma \rho^*)
\label{e:A-eig-values}
\end{equation}
Consequently, the matrix $A$ is Hurwitz, for all $0 < \gamma < 1/\rho^*$.
\qed
\end{lemma}

\begin{proof}[Proof of \Theorem{t:RelQADgain}]
\Lemma{t:A-eig-values} proves the first conclusion in \Theorem{t:RelQADgain}: Each eigenvalue of the linearization matrix $A$ of relative Q-learning satisfies $\Real (\lambda(A)) \leq - (1  - \gamma \rho^*)$. The proof of  (i) and~(ii) then follow from \Proposition{t:SAnonlinear}.
\end{proof}

%
%
%

\section{Discussion}
\label{s:disc}



Theorems~\ref{t:ADgain} and~\ref{t:RelQADgain} containx conditions for finite asymptotic covariance of the Q-learning and relative Q-learning algorithms.  Here we provide a more quantitative comparison.   We begin with a coarse comparison, considering the trace of the respective covariance matrices.

\begin{proposition}
\label{t:traceSigmas} 
Denote by $ \Sigma_\theta^q (g)$,   $ \Sigma_\theta^h(g)$,  the asymptotic covariance matrices for Q-learning and relative Q-learning with step-size $\alpha_n = g \cdot [n(x,u)]^{-1}$.    Each is finite for all sufficiently large $g$, and satisfy the following bounds, uniformly in $\gamma$: 
\begin{subequations}
\label{e:traceBdds}
\begin{align}
\begin{split}
\hspace{-0.07in}\min_g  \bigg\{ \! \trace \bigl( \Sigma_\theta^q (g) \bigr)  \! \biggr\} & \! \ge \!  O \bigg( \! \frac{\trace(\Sigma_\Delta^2)}{ (1-\gamma )^2} \! \bigg),\,  
			\ \ \textit{sub.\ to \eqref{e:hvar}}
\label{e:traceBdds-q}
\end{split}
\\
\begin{split}
\hspace{-0.1in}\min_g \bigg\{ \! \trace \bigl( \Sigma_\theta^h (g) \bigr) \! \biggr\} & \! \le \! O \bigg( \! \frac{\trace( \Sigma_\Delta^2)}{(1- \rho^* \gamma )^2} \! \bigg)
\label{e:traceBdds-h}
\end{split}
\end{align}
\end{subequations} 
\end{proposition}

\begin{proof}
The proof of eqn.~\eqref{e:traceBdds-q} follows from \Proposition{t:SigmaSingleSubspace-a}.  This result also implies  that the lower bound for $\Sigma_\theta^q (g) $  in \eqref{e:traceBdds-q} is attained with $g_q = (1 - \gamma  )^{-1}$.
The upper bound \eqref{e:traceBdds-h} is a simple consequence of \Theorem{t:RelQADgain}.
\end{proof}

These bounds show a significant contrast in performance when $1/(1-\gamma) \gg 1/(1-\rho^*)$.   However,  we find that the two covariance matrices actually coincide on a subspace.   This is made precise in the following subsections.

\subsection{Covariance  Comparison on a Single Eigenspace}
\label{sec:BadSigmas}

x
We first amplify the stark contrast between the two covariance matrices. 
Denote by  $\lambda_{h_1}$ the eigenvalue of the matrix $A_h = A$ (defined in \eqref{e:Amatrix}) that has the largest real part (marked with pink circles in the third part of \Fig{f:Beta}), and $\nu_{h_1}$ the corresponding left-eigenvector. Similarly, denote $\lambda_{q_1}$ to be the eigenvalue of $A_q$ (defined in \eqref{e:AQmatrix}) that has the largest real part (the green circle in the second part of \Fig{f:Beta}), and $\nu_{q_1}$ the corresponding left-eigenvector. Our interest here is the magnitude of  the non-negative quantities
\begin{equation} 
\sigma_q^2(1, 1)  \eqdef \nu_{q_1}^\dagger \Sigma_\theta^q \nu_{q_1} \qquad \textit{and}
\qquad
\sigma_h^2(1, 1) \eqdef \nu_{h_1}^\dagger \Sigma_\theta^h \nu_{h_1} 
\label{e:sigmaq2_sigmah2_11}
\end{equation} 
Explicit formulae are easily obtained, and then optimized over $g$.   Analogous formulae are obtained in \Section{s:diagSigma} for other eigenvectors, so we omit the proof of \eqref{e:sigma_qh2_one}
here.

\begin{proposition}  
\label{t:SigmaSingleSubspace-a}
\begin{subequations}
\label{e:sigma_qh2_one}
\begin{align}
\begin{split}
\sigma_q^2(1,1) \!=\! g^2 \frac{ \sigma_{\Delta_q}^2(1,1)}{1 - 2 g (1-\gamma) }\,, \,\,\,\,\, g> [2(1-\gamma)]^{-1} 
\label{e:sigma_q2_one}
\end{split}
\\
\begin{split}
\sigma_h^2(1,1) \!=\!   g^2\frac{ \sigma_{\Delta_h}^2(1,1)}{1 - 2 g (1-\gamma \rho^*)    }\,, \,\,\,\,\,  g> [2 (1-\gamma \rho^*)  ]^{-1} 
\label{e:sigma_h2_one}
\end{split}
\end{align}
\end{subequations}
where $
\sigma_{\Delta_q}^2(1, 1) \eqdef \nu_{q_1}^\dagger \Sigma_\Delta \nu_{q_1}  
$
and
$
\sigma_{\Delta_h}^2(1, 1) \eqdef \nu_{h_1}^\dagger \Sigma_\Delta \nu_{h_1}$.   
The minimizing gains are given by $  g_q = (1-\gamma)^{-1}$ in \eqref{e:sigma_q2_one},  and $g_h = (1-\gamma \rho^*)^{-1}$ in \eqref{e:sigma_h2_one}.
This results in the minimal values,
\begin{align}
\min_g 
\sigma_q^2(1,1) \!=\!  \frac{\sigma_{\Delta_q}^2 (1,1) }{(1-\gamma)^2}\,, \,\,\, 
\min_g \sigma_h^2(1,1) \!=\!  \frac{\sigma_{\Delta_h}^2 (1,1) }{(1-\rho^* \gamma)^2}
\label{e:sigma_qandh2_one}
\end{align}
\qed
\end{proposition}
The optimized step-size scaling suggested in \Proposition{t:SigmaSingleSubspace-a} ensures that all eigenvalues of the matrices $g_q A_q$ and $g_h A_h$ have real parts $\leq - 1$:
\[
\text{Re}(\lambda(g_q A_q)) \leq -1 \qquad \text{Re}(\lambda(g_h A_h)) \leq -1 
\]

Denote by $(\lambda_{h\one}, \nu_{h\one})$ the eigenvalue/left-eigenvector pair of the matrix $A_h$, with corresponding right-eigenvector $\one$. Similarly, denote by $(\lambda_{q\one}, \nu_{q\one})$ the eigenvalue/left-eigenvector pair of the matrix $A_q$, with corresponding right-eigenvector $\one$. Note that these eigenvalues correspond to the green circles in \Figure{f:Beta}, and $(\lambda_{q\one}, \nu_{q\one})$ coincides with $(\lambda_{q_1}, \nu_{q_1})$, but a similar property \emph{does not} hold for the matrix $A_h$.  We next compare the covariance of the algorithms on the eigenspace that corresponds to eigenvector $\one$:
\begin{equation} 
\sigma_{q\one}^2  \eqdef \nu_{q\one}^\dagger \Sigma_\theta^q \nu_{q\one} \qquad \textit{and}
\qquad
\sigma_{h\one}^2 \eqdef \nu_{h\one}^\dagger \Sigma_\theta^h \nu_{h\one} 
\label{e:sigmaq2_sigmah2_one}
\end{equation} 
\begin{proposition}  
\label{t:SigmaSingleSubspace-b}
Consider the Q-learning and relative Q-learning with step-size scaling $g_q$ and $g_h$ defined in \Proposition{t:SigmaSingleSubspace-a}, and with $\delta \geq \gamma(1-\rho^*)$ in \eqref{e:eligHalg}. Then,
\label{e:sigma_qh2_one_vec}
\begin{align}
\hspace{-0.07in}\sigma_{q\one}^2  \!=\! \frac{\sigma_{\Delta_{q\one}}^2 }{(1-\gamma)^2}
\quad \,\,\,\,\,
\sigma_{h\one}^2  \!=\!  \frac{\sigma_{\Delta_h (\one)}^2}{(1 \! - \! \rho^* \gamma) \big(1 \! + \! \rho^* \gamma \! - \! 2 ( \gamma \! - \!  \delta)  \big)}
\label{e:sigma_q2_and_h2_one_vec}
\end{align}
where,
$
\sigma_{\Delta_{q\one}}^2 \!=\! \sigma_{\Delta_q}^2(1, 1) \!=\! \nu_{q\one}^\dagger \Sigma_\Delta \nu_{q\one}$, 
$\sigma_{\Delta_h (\one)}^2 \! = \! \nu_{h\one}^\dagger \Sigma_\Delta \nu_{h\one}
$.
\qed
\end{proposition}
Note that for the choice of $\delta \geq \gamma (1- \rho^*)$, we have $\sigma_{h\one}^2 \leq \sigma_h^2(1, 1)$ defined in \eqref{e:sigma_h2_one}, consistent with \eqref{e:traceBdds-h} of \Proposition{t:traceSigmas}.

In \Fig{f:Hplot} we compare the performance of Q-learning and relative Q-learning algorithms applied to a simple $6$-state~MDP that was considered in \cite[Section~3]{devmey17b}. Experiments were run for $\gamma = 0.999$ and $\gamma = 0.9999$, and in each of the two cases, we implemented Q-learning with optimized step-size $\alpha_n = g_q / n$, $g_q = 1/(1-\gamma)$, and relative Q-learning with optimized step-size $\alpha_n = g_h / n$, $g_h = 1/(1-\rho^* \gamma)$. In addition, we also implemented Q-learning with $\alpha_n = g_h / n$; the motivation is discussed in the following subsection.

Histogram of $\{ \sqrt{N} \tiltheta_N(i) \}$ for $10^3$ independent runs.   The CLT approximation is good even for the shortest run, and nearly perfect for $N\ge 10^4$.

We return now to 
\Figure{f:relW10CLT}, which shows histograms of $\{ \sqrt{N} \tiltheta_N(i) \}$ from the relative Q-learning algorithm. 
The histograms were obtained by running $10^3$ independent runs of the algorithm with random initial conditions, up to time horizon $10^6$  (with data collected at this value, and intermediate values
 $N= 10^3$, $10^4$, $10^5$).
 The theoretical pdf's were obtained based on CLT~\eqref{e:SACLT}: For $N$ \emph{large enough}, the distribution of $\sqrt{N}\tiltheta_N$ is approximated by $\clN(0,\Sigma_{\theta})$, where $\Sigma_{\theta}$ is obtained as a solution to \eqref{e:Lyapg}.   The figure makes clear that the CLT predicts finite-$N$ behavior for $N$ as small as $10^4$.   This is remarkable, but not surprising given the success in prior RL studies \cite{devmey17b,devbusmey19}.


\subsection{Solidarity on a Subspace}
\label{sec:sol_sub}

\Proposition{t:SigmaSingleSubspace-a} again shows that a larger gain $g$ is required in Watkins' algorithm,   and we can expect a larger asymptotic covariance.  \Proposition{t:SigmaSingleSubspace-b} compares the asymptotic covariance with optimized $g$'s for the two algorithms on a particular subspace. The question we ask here is: \emph{what about the remainder of $\Re^d$?}

The asymptotic covariances appearing in \Prop{t:traceSigmas}    solve the respective Lyapunov equations:
\begin{subequations}
\label{e:Lyap12g}
\begin{align}
0  &=  F_q \Sigma_\theta^q+ \Sigma_\theta^q F_q^\transpose  + g^2 \Sigma_\Delta
\label{e:SigmaThetaQ}
\\
0          & = F_h  \Sigma_\theta^h+ \Sigma_\theta^h F_h^\transpose  + g^2 \Sigma_\Delta
\label{e:SigmaThetaH}
\end{align} 
\end{subequations}
where $F_q = g A_q +\half I$ and  $F_h = g A_h +\half I$.  It is shown in \Prop{t:SigmaEqual}
that
the solutions are identical on the subspace 
\[
\Re_0^d = \{ v \in \Re^d : v^\dagger \one = 0 \}
\]
\vspace{-0.03in}
in the sense that 
\begin{equation}
v^\dagger  \Sigma_\theta^q  w  =  
v^\dagger  \Sigma_\theta^h  w  \,, \qquad \textit{for all $v,w\in\Re_0^d$}
\label{e:SigmaSolid}
\end{equation}
This identity is valid even when $F_q$  \textit{is not Hurwitz}, so that $\Sigma_\theta^q $ is not finite valued.    To make this precise we   make use of the representations
\begin{equation}
\begin{aligned}
v^\dagger\Sigma_\theta^q  w & =   g^2 \int_0^\infty v^\dagger e^{F_q t} \Sigma_\Delta {e^{F_q^\transpose t}} w \, dt
\\
v^\dagger\Sigma_\theta^h  w & =   g^2 \int_0^\infty v^\dagger e^{F_h t} \Sigma_\Delta {e^{F_h^\transpose t}} w \, dt
\label{e:Frep}
\end{aligned}
\end{equation}

We do not assume that $F_q$ is Hurwitz in \Prop{t:SigmaEqual}, so that $\Sigma^q_\theta$ may not be finite valued.

\begin{proposition}
\label{t:SigmaEqual}
Suppose that the matrix $F_h$ is Hurwitz.    Then the asymptotic covariance $\Sigma_\theta^h$ exists and is finite,  and moreover \eqref{e:SigmaSolid} holds, subject to the definition of \eqref{e:Frep}.
\end{proposition} 

\begin{proof}
Given the representation \eqref{e:Frep}, it is enough to establish
\begin{equation}
v^\dagger  e^{t F_q }   =v^\dagger  e^{t F_h }       \qquad \textit{for   all $t>0$ and $v\in\Re_0^d$}
\label{e:SolidSolid}
\end{equation}
The proof makes use of  the following identity:
\begin{equation}
v^\dagger F_q = v^\dagger F_h   \,, \qquad \textit{for all $v \in\Re_0^d$}
\label{e:Solid}
\end{equation}
Moreover,  $F_q^\dagger \colon\Re_0^d\to\Re_0^d$   and $F_h^\dagger \colon\Re_0^d\to\Re_0^d$.

These identities imply many others.   
Starting from  $v^\dagger F_q = v^\dagger F_h$  for $v\in\Re_0^d$, we obtain $v^\dagger F_q F_h = v^\dagger F_h^2$,   and  the identity $v^\dagger F_q^2 = v^\dagger F_h^2$ follows since 
$(v^\dagger F_q)^\dagger \in\Re_0^d$. By induction we obtain $v^\dagger F_q^n = v^\dagger F_h^n$ for each $n$ and each $v\in\Re_0^d$, and then \eqref{e:SolidSolid} follows from the Taylor series representation of the matrix exponential.
\end{proof}

In \Fig{f:HplotSpanSemiNorm} we plot the span-semi-norm of the errors in the Q-function estimates obtained using Q-learning and relative Q-learning algorithms. Once again, experiments were run for $\gamma = 0.999$ and $\gamma = 0.9999$, and for each $\gamma$, we used two different step-sizes for the Q-learning algorithm: $g_q = 1/(1-\gamma)$, and $g_h = 1/(1-\rho^* \gamma)$. Since the span semi-norm ignores the error in the constants, we notice that the performance of the Q-learning and relative Q-learning algorithms with the same step-size $g_h = 1/(1-\rho^* \gamma)$ is very similar --- consistent with our findings in \Proposition{t:SigmaEqual}.

\begin{figure}[htbp]
\centering

	\includegraphics[width=\hsize]{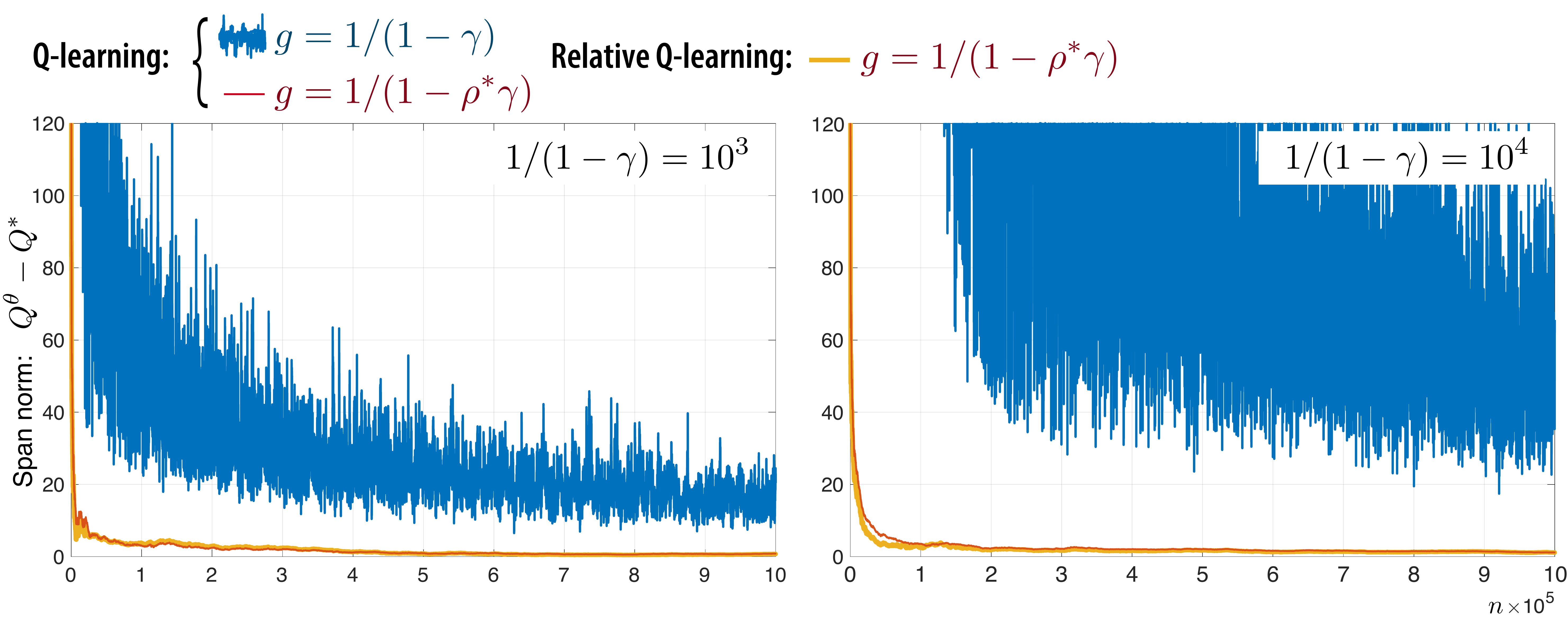}

\caption{Comparison of Q-learning and Relative Q-learning algorithms (in terms of the span-norm of the error) for the stochastic shortest path problem of \cite{devmey17a}. Provided we choose the right step-size, relative Q-learning and Q-learning have similar performance with this metric.}
\label{f:HplotSpanSemiNorm}
\end{figure}

\subsection{What if the Transition Matrix is Diagonalizable?}
\label{s:diagSigma}

If the matrix $PS_\phi^*$ is diagonalizable,  this means that there is a basis consisting of eigenvectors,  and also a basis consisting of left-eigenvectors.   Viewed as column vectors,  we find that $d-1$ of the left eigenvectors span $\Re_0^d$.    From this we obtain  a refinement of \Prop{t:SigmaEqual}:  a solution to the Lyapunov equation on $\Re_0^d$,   and on all of $\Re^d$ when $F_q$ is Hurwitz.  

If $PS_\phi^*$ is diagonalizable, then the definition \eqref{e:AQmatrix}  implies that the same is true for $A_q$.  Let $\{\nu_i: 1 \leq i \leq d\}$ be a basis of left eigenvectors for $A_q$,   with corresponding eigenvalues $\{\lambda_i: 1 \leq i \leq d\}$, and suppose the eigenvalues are ordered so that $\lambda_1(A_q) = -(1-\gamma)$.

\Lemma{t:AQ-lemma}~(ii) asserts that $\{\nu_i: 2 \leq i \leq d\}$  are also left eigenvectors for $A_h$, with common left eigenvalues.   Moreover,   $\nu_i^\dagger \one = 0$ for $2\le i\le d$, so that   the span of these vectors is precisely  $\Re^d_0$.

For each  $2 \leq i,j \leq d$, consider the quantities:
\begin{equation}
\begin{aligned}
\sigma_q^2(i, j) \eqdef  \nu_i^\dagger \Sigma_\theta^q \nu_j  & \qquad
\sigma_h^2(i, j) \eqdef \nu_i^\dagger \Sigma_\theta^h \nu_j
\\
\sigma_\Delta^2(i, j) & \eqdef \nu_i^\dagger \Sigma_\Delta \nu_j 
\label{e:sigmaq2_sigmah2}
\end{aligned}
\end{equation} 
The identity $\sigma_q^2(i, j) = \sigma_h^2(i, j) $ follows from \Proposition{t:SigmaEqual}.   
Multiplying the left hand side of \eqref{e:SigmaThetaQ} and \eqref{e:SigmaThetaH} by $\nu_i^\dagger$, and the right hand side by $\nu_j$, we obtain 
\begin{equation}   
\sigma_q^2(i,j) = \sigma_h^2(i,j)   =  g^2 \frac{  \sigma_\Delta^2 (i, j)}{   1 - g  (\lambda_i + \lambda_j ) }
\label{e:sigma_q2h2} 
\end{equation}
For the optimal gains $g_q$ and $g_h$ appearing in 
\Proposition{t:SigmaSingleSubspace-a},
substitution into \eqref{e:sigma_q2h2} gives the approximation when $\gamma \approx 1$: For $2 \leq i,j \leq d$,
\begin{align}
\sigma_q^2(i,j)  \! = \! O \bigg( \frac{\sigma_\Delta^2 (i, j) }{1-\gamma} \bigg)\,,
\quad \sigma_h^2(i,j)  \!=\!  O \bigg( \frac{\sigma_\Delta^2 (i, j)}{1-\rho^* \gamma} \bigg)
\label{e:sigma_qandh2_notone}
\end{align}

\section{Conclusions and Future Work}

The factor $1/(1-\gamma)^p$ is ubiquitous in RL complexity bounds, where $p \geq 2$. We have shown that this dependency is \emph{artificial}: if we ignore the constant terms (that does not affect the optimal policy), this factor can be improved to $1/(1- \rho^* \gamma)^p$, where $\rho^* < 1$ under very general conditions. Specifically, we showed that the classical Q-learning algorithm of Watkins has asymptotic (CLT) variance that grows as a quadratic in $1/(1-\gamma)$, and the relative Q-learning algorithm has asymptotic variance that is bounded by a quadratic in $1/(1-\rho^* \gamma)$.  We believe that this will lead to comparable improvements in sample complexity bounds.

The techniques introduced in this work can also be extended to various other RL algorithms. For example, it is straightforward to modify the recursion \eqref{e:DevMey} to obtain a \emph{relative TD($0$)-learning} algorithm for a discounted cost MDP, with linear function-approximation.  The choice of $\mu$ may require care in a  continuous state-space setting (perhaps an empirical distribution is preferable).


The ideas introduced in this paper are complementary to the Zap-Q techniques of \cite{devmey17b}. In view of the matrix gain Q-learning algorithm \eqref{e:HQlambda}, the goal in \cite{devmey17b} is to obtain an optimal matrix gain sequence $\{G_{n+1}\}$ that will result in minimum asymptotic covariance $\Sigma_{\theta}$. It is straightforward to \emph{Zap} our relative Q-learning algorithm, resulting in a further reduction of asymptotic covariance.

We close with three open problems:
\begin{romannum}
\item How do we choose $\delta$? It seems that the choice $\delta = 1$ will serve our purpose of uniformly bounding the asymptotic variance of Q-learning. Perhaps a larger $\delta$ will result in better transient behavior?

\item   Is there an optimal  choice for $\mu$  (in terms of both variance and transient behavior)?

\item  How can these ideas extend to Q-learning outside of the tabular setting?

\end{romannum}

\def\cprime{$'$}\def\cprime{$'$}

\newpage

\appendix
\begin{center}
\Huge{\bf Appendix}
\end{center}

\section{Convergence Rate of Nonlinear Stochastic Approximation}
\label{s:SAnonlinearProof}

\begin{proof}[Proof of \Proposition{t:SAnonlinear}]
Part (i) of the Proposition follows from the main result of \cite[Ch.~7]{bor20a}.  

The proof of (ii) uses similar ideas as in \cite{chedevbusmey20}.
For simplicity we normalize so that $\theta^*=0$,  and take $g=1$ so that $\alpha_{n} =1/n$.     

The proof  proceeds by contradiction:  Suppose that   $ n^{2\varrho_1} \Expect[\|\tiltheta_n\|^2]$ is bounded in $n$ for some $\varrho_1 >  \varrho_0$, and consequently $ n^{2\varrho} \Expect[\|\tiltheta_n\|^2]$ tends to zero as $n\to\infty$, for any   $\varrho_1 >\varrho >  \varrho_0$.  
We then use the new definition  $W_n = n^{\varrho}  \theta_n$,  and denote, for a fixed $\theta \in \Re^d$,
\[
\barf_n(\theta) = (n+1)^{\varrho}   \barf( n^{-\varrho} \theta)   \,,\quad     \Upsilon_n =  \alpha_n n^{\varrho} \Delta_n
\]
On multiplying each side of \eqref{e:SAintro} by $(n+1)^{\varrho} $ we obtain
\[ 
W_{n+1}  = W_n +\alpha_{n+1} [ \varrho_n W_n +  \barf_n (W_n)   ]  +  \Upsilon_{n+1} 
\]
where $ \varrho_n = \varrho + o(1)$ appears through the Taylor series approximation $ (n+1)^{\varrho}  =    n^{\varrho}   +  \varrho \alpha_{n+1} n^{\varrho}  + o(\alpha_{n+1} )$.

Under the \textit{assumption} that  $ n^{2\varrho_1} \Expect[\|\tiltheta_n\|^2]$ is bounded in $n$,  it follows that 
\[
\lim_{n\to\infty} 
\Expect[  \|  \barf_n (W_n) -    \barf_0 (W_n) \|^2 ] =0
\]
and from this we obtain the approximately linear recursion for $\Sigma_n^W = \Expect[W_n W_n^\transpose]$:
\[ 
\Sigma^W_{n+1}  = \Sigma_n^W  +\alpha_{n+1} [   (\varrho + A) \Sigma_n^W +\Sigma_n^W  (\varrho + A)^\transpose 
+ \clE_n]   +  \Sigma^\Upsilon_{n+1} 
\]
where the vanishing sequence $\{\clE_n\}$ is composed of three approximations:  replacing $\varrho_n$ by $\varrho$,  the replacement of $\barf_n$ by $\barf_0$,   and the final term:
\[
\alpha_{n+1}^2 \Expect [  ( \varrho_n W_n +  \barf_n (W_n)  ) ( \varrho_n W_n +  \barf_n (W_n)  )^\transpose ] 
\]

Denote  $\sigma^2_n = n^{2\varrho}  \Expect[ |\nu^\transpose  W_n|^2]   =  \nu^\transpose \Sigma_n^W  \nu$, which is a vanishing sequence by assumption.    
However, it   evolves according to the recursion 
\[ 
\sigma^2_{n+1}  =  \sigma^2_n + \alpha_{n+1} [  2 (\varrho -\varrho_0) \sigma^2_n   
+  \nu^\transpose  \clE_n  \nu]   +      \nu^\transpose  \Sigma^\Upsilon_{n+1}  \nu
\]
As in  \cite{chedevbusmey20}, this can be regarded as a deterministic SA recursion, and apparently unstable under our assumption that $  \varrho -\varrho_0>0$.  This can be verified using an ODE approximation:  $ \sigma^2_n \to\infty$ as $n\to\infty$ under this assumption, along with the fact that $ \nu^\transpose  \Sigma^\Upsilon_{n+1}  \nu >0$ for at least one $n$ (recall that $\Sigma_\Delta \nu \neq 0$.  This contradiction completes the proof.  
\end{proof}

\section{ODE Approximation of Q-learning}
\label{sec:Q_appendix}

\begin{proof}[Proof of \Lemma{t:WatkinsLin}]
To prove (i) we recall the definition of the ODE \eqref{e:QODEW}:  $\ddt q_t = \barf(q_t)$, where for any $q\colon\state\times\U\to\Re$, 
and $1\le i\le d$,
\[
\barf_i(q) = \Expect\bigl[ \bigl\{  c(X_n,U_n)   + \gamma     \uq (X_{n+1})  - q(X_n,U_n)  \bigr\}   \psi_i(X_n, U_n)\bigr]
\]
Substituting $ \uq (X_{n+1})  = q (X_{n+1},   \phi^*(X_{n+1} ))$   for  $\| q - Q^*\| <\epsy$ gives
\[
\begin{aligned}
\barf_i(q) &= \Expect\bigl[   c(X_n,U_n)       \psi_i(X_n, U_n)\bigr]  
\\
& \hspace{0cm}+
\Expect \big [ \psi_i (X_n,U_n) \{  \gamma    q (X_{n+1}, \phi^*(X_{n+1} ) ) - q(X_n,U_n)  \bigr\}    \big]    
\\
& =  \pie(x^i,u^i) c(x^i,u^i)   
\\
& +  \pie(x^i,u^i) \bigl \{ \gamma \sum_j    P_{u^i}(x^i, x^j)  q(x^j,\phi^*(x^j)   - q(x^i,u^i)    \bigr\}
\end{aligned} 
\]
where the second identity follows from the tabular basis.  This
establishes (i).

Part (ii) is immediate, given the similarity of the two ODEs.  
\end{proof}

\section{Convergence Analysis of Relative Q-learning}
\label{sec:conv_appendix}

\begin{proposition}
The ODE \eqref{e:RelQODE_Gain} is globally asymptotically stable, with unique equilibrium $H^*$.
\label{t:ODEisstable}
\end{proposition}

For any function $H: \state \times \U\to \Re$, define the span semi-norm:
\begin{equation}
\| H \|_S \eqdef \max_{x \,, u} H(x, u) - \min_{x \,, u} H(x, u)
\label{e:span}
\end{equation}
The crucial step in proving stability of the ODE \eqref{e:RelQODE_Gain} is based on the fact that the operator $\tilT$ defined in \eqref{e:TilT} is a $\gamma$-contraction in the span semi-norm: for any $H \,, H' : {\state \times \U} \to \Re$, 
\begin{equation}
\| \tilT H - \tilT H' \|_S \leq \gamma \| H - H' \|_S
\label{e:TilTcontraction}
\end{equation}
This is formalized in the following Lemma.
\begin{lemma}
\label{t:span-norm-contraction}
For any $0 \leq \gamma < 1$, the operator $\tilT$ is a $\gamma $-contraction: For any $H: {\state \times \U} \to \Re$ and $H' : {\state \times \U} \to \Re$,
\[
\| \tilT H - \tilT H' \|_S \leq \gamma \| H - H' \|_S
\]
\end{lemma}
The proof is trivial --- see for example \cite{put14}.

The contraction property \eqref{e:TilTcontraction} is used next to prove the stability of the ODE \eqref{e:RelQODE_Gain} in the span semi-norm.

For each $t \geq 0$, define $\tilh_t \eqdef h_t - H^*$, where $H^*$ is the unique solution to the fixed point equation \eqref{e:DCOE-H}.
Define  
\begin{equation*}
\begin{aligned}
\phi_t^* & \eqdef \phi^{(\kappa)}\,\,\text{such that}
\\
\kappa & = \min \{ i : \phi^{(i)} (x) \in\argmin_u h_t  (x,u), \,\, \text{for all }x \in \state  \}
\end{aligned}
\end{equation*}
The following proposition establishes exponential convergence of $h_t$ to $H^*$ in the span semi-norm, which further implies the same rate of convergence for $H^*(x,\phi_t^*(x))$ to $H^*(x,\phi^*(x))$, for each $x \in \state$.
\begin{proposition}
\label{t:span-norm-stability}
For any $0 \leq \gamma < 1$, and some $K < \infty$,
\begin{subequations}
\begin{align} 
\| \tilh_t \|_S & \leq e^{-(1 - \gamma) t} \| \tilh_0 \|_S  
\label{e:tilht}
\\
\|H^*(x, \phi_t^*(x) ) \! - \!  H^*(x, \phi_t^*(x) ) \| & \leq K e^{-(1-\gamma)t}\|\tilh_0\|
\label{e:tilphit}
\end{align}
\end{subequations}
\end{proposition}

\begin{proof}
By the variations of constants formula, and using the notation \eqref{e:TilT}, the solution $h_t$ to \eqref{e:RelQODE_Gain} satisfies:
\[
h_t(x \,, u) =  h_0(x \,, u) e^{-t} + \int_{0}^t e^{-(t-s)} (\tilT h_s)(x \,, u) \,ds
\]
Subtracting $H^*(x \,, u)$ from both sides, and using the fact that $H^*(x \,, u) = (\tilT H^*)(x \,, u) $, we obtain
\begin{equation*}
\begin{aligned}
\tilh_t(x\,,u) & = \tilh_0(x,u) e^{-t} 
\\
& \hspace{0.1in}+ \int_{0}^t e^{-(t-s)} \Big[ (\tilT  h_s)(x \,, u) - (\tilT  H^*)(x \,, u)\Big] ds
\end{aligned}
\end{equation*}
The following inequalities are then immediate
\begin{subequations}
\begin{align}
\begin{split}
\max_{x \,, u} \tilh_t(x\, ,u) & \leq \max_{x \,, u} \tilh_0(x\, ,u) e^{-t} 
\\
& \hspace{-0.65in} + \int_{0}^t e^{-(t-s)} \max_{x \,, u}  \Big[ (\tilT  h_s)(x \,, u) \! - \! ( \tilT  H^*)(x \,, u)\Big] ds
\label{e:Ht-error-ODE-upper-bd} 
\end{split}
\\
\begin{split}
\min_{x \,, u} \tilh_t(x\, ,u) & \geq \min_{x \,, u} \tilh_0(x\, ,u) e^{-t} 
\\
& \hspace{-0.65in} + \int_{0}^t e^{-(t-s)} \min_{x \,, u}  \Big[ (\tilT  h_s)(x \,, u) \! - \! (\tilT  H^*)(x \,, u)\Big] ds
\label{e:Ht-error-ODE-lower-bd} 
\end{split}
\end{align}
\label{e:Ht-error-ODE-bds} 
\end{subequations}
Subtracting \eqref{e:Ht-error-ODE-lower-bd} from \eqref{e:Ht-error-ODE-upper-bd},
\begin{equation}
\begin{aligned}
\| \tilh_t \|_S & \leq e^{-t}  \| \tilh_0 \|_S 
+ \int_{0}^t e^{-(t-s)} \| \tilT  h_s - \tilT  H^* \|_S \, ds
\\
& \leq e^{-t}  \| \tilh_0 \|_S 
+ \gamma \int_{0}^t e^{-(t-s)} \| \tilh_s \|_S \, ds
\end{aligned}
\end{equation}
where the second inequality follows from \eqref{e:TilTcontraction}. We therefore have
\[
e^{t} \| \tilh_t \|_S  \leq \| \tilh_0 \|_S 
+ \gamma \int_{0}^t e^{s} \|  \tilh_s \|_S \, ds
\]
Applying the Gr\"onwall's inequality completes the proof of \eqref{e:tilht}; \eqref{e:tilphit} follows from \eqref{e:phi_star_H} and \eqref{e:tilht}.
\end{proof}

Define for each $t \geq 0$
\begin{equation}
r_t \eqdef \langle \mu \,, \tilh_t \rangle
\label{e:rt}
\end{equation}
\Proposition{t:span-norm-stability} (in particular, Eq.~\eqref{e:tilht}) implies $ h_t  \to H^*$ exponentially fast, in the span-semi-norm:
for some $K < \infty$,
\begin{equation}
\begin{aligned}
\tilh_t & = \one \cdot r_t + \epsy^s_t
\\
\| \epsy^s_t \| & \leq K e^{-(1 - \gamma) t} \| \tilh_0 \|_S     
\label{e:ht_Hstar}
\end{aligned}
\end{equation}
To establish global exponential stability of the ODE \eqref{e:RelQODE_Gain}, it is sufficient to show that $r_t \to 0$ exponentially fast.
\begin{proposition}
\label{t:constant-stability}
For any $\gamma <1$, and $\delta > 0$, the function $r_t$ defined in \eqref{e:rt} satisfies, for some $K < \infty$,
\[
| r_t | \leq e^{- (1 - \gamma + \delta )} | r_0 | + K e^{-(1-\gamma) t} \| \tilh_0 \|
\]
\end{proposition}
\begin{proof}
Differentiating both sides of \eqref{e:rt}, and using \eqref{e:RelQODE_Gain}, we have
\begin{equation}
\begin{aligned}
\ddt r_t  = \langle \mu \,,  \ddt h_t \rangle 
& = \langle \mu \,, \tilT h_t  - h_t \rangle
\\
& =  \langle \mu \,,  \tilT h_t - h_t - \tilT H^* + H^* \rangle
\end{aligned}
\label{e:rt_ODE}
\end{equation}
where we have used the fact that $H^* = \tilT H^*$.

Using \eqref{e:ht_Hstar} and \eqref{e:tilphit},
the non-linear term on the right hand side of \eqref{e:rt_ODE} admits the approximation,
\[
\begin{aligned}
& \big \langle \mu \,, \tilT h_t - \tilT H^* \big \rangle 
\\
& =  \gamma \sum_{x,u,x'} \mu(x,u)  P_{u} (x,x') \Big[ h_t \big(x', \phi^\epsy_t(x') \big) \! - \! H^* \big(x', \phi^*(x') \big) \Big  ] 
\\
& \hspace{1in}- \delta \cdot \langle \mu \,, \tilh_t \rangle
\\
& = \gamma \sum_{x,u,x'} \mu(x,u)  P_{u} (x,x')  \!  \Big[  \!  H^* \big(x', \phi^\epsy_t(x') \big) \!  - \! H^* \big(x', \phi^*(x') \big) 
\\
&  \hspace{1in} +  \! \epsy^s_t\big(x', \phi^\epsy_t(x') \big) \!  \Big  ]  \!  +  \!  \big (\gamma  \!  -  \!  \delta \big) \cdot r_t
\\
&  =  \big (\gamma - \delta \big) \cdot r_t + \epsy^r_t  
\end{aligned}
\]
where, for some $K < \infty $,
\[
| \epsy^r_t | \leq K e^{-(1-\gamma)t}\|\tilh_0\|  
\]
Substituting this into \eqref{e:rt_ODE}, we have
\[
\begin{aligned}
\ddt r_t & = \big (\gamma - \delta - 1 \big) \cdot r_t + \epsy^r_t,
\\
r_t & = e^{- (1 - \gamma + \delta )} \cdot r_0 + \int_{0}^{t} e^{-(1-\gamma + \delta) \tau} \cdot \epsy^r_{t-\tau} \, d \tau
\end{aligned}
\]
The statement of the proposition follows.
\end{proof}

Propositions~\ref{t:span-norm-stability} and~\ref{t:constant-stability} imply the conclusions of Propositions~\ref{t:ODEisstable}. 
\qed


\section{Convergence Rate of Relative Q-learning}
\label{sec:asym_var_appendix}

\begin{lemma}
\label{t:orth-lemma}
Let $P$ be a  transition matrix with a single eigenvalue at $\lambda=1$,  with unique invariant measure $\pi$.   Then,  
$\nu^{\transpose} \one= 0$ for every eigenvalue/left-eigenvector pair $(\lambda,\nu)$  of $P $ for which $\lambda\neq 1$.
\end{lemma}
\begin{proof} 
The eigenspace corresponding to the eigenvalue $\lambda=1$ is spanned by the vector $\one$, so that
\[
P \one = \one \quad \textit{and}\quad \nu^\transpose P = \nu^\transpose
\]
Consequently,
\begin{equation*}
\lambda \nu^\transpose \one =
\nu^\transpose P \one = \nu^\transpose \one \,,  
\end{equation*}
which   implies that $\nu^\transpose \one = 0$.
\end{proof}

\end{document}